\documentclass[11pt]{article}

\usepackage[margin=1in]{geometry}
\usepackage{setspace}
\usepackage{amsmath,amssymb,amsthm}
\RequirePackage{tgtermes}
\RequirePackage{newtxtext}
\RequirePackage{newtxmath}
\usepackage{bm}
\usepackage{endnotes}

\usepackage{algorithm}
\usepackage{algpseudocode}
\usepackage{tikz}
\usepackage[utf8]{inputenc}
\usepackage[T1]{fontenc}
\usepackage{hyperref}
\usepackage{url}
\usepackage{booktabs}
\usepackage{amsfonts}
\usepackage{nicefrac}
\usepackage{microtype}
\usepackage{xcolor}
\usepackage{graphicx}
\usepackage{comment}
\usepackage{array}
\newcolumntype{L}{>{\raggedright\arraybackslash}p{1.5cm}}
\newcolumntype{C}{>{\centering\arraybackslash}p{1.5cm}}
\newcolumntype{R}{>{\raggedleft\arraybackslash}p{1.1cm}}

\usepackage{natbib}
\bibpunct[, ]{(}{)}{,}{a}{}{,}%
\newtheorem{theorem}{Theorem}

\newcommand{\TABLE}[3]{%
  \centering
  \caption{#1}
  #2\par
  \vspace{4pt}
  {\footnotesize #3}
}

\title{(Mis)Understanding Benign Overfitting \\ in Equity Return Prediction}

\author{%
  Hui Guo\thanks{Carl H. Lindner College of Business, University of Cincinnati, \texttt{guohu@ucmail.uc.edu}}
  \and
  Jiawei Huang\thanks{Cat\'olica Lisbon School of Business and Economics, Universidade Cat\'olica Portuguesa, \texttt{jhuang@ucp.pt}}
  \and
  Runze Li\thanks{Department of Statistics, Pennsylvania State University, \texttt{rzli@psu.edu}}
  \and
  Yan Yu\thanks{Carl H. Lindner College of Business, University of Cincinnati, \texttt{Yan.Yu@uc.edu}}
}

\date{}

\begin{document}

\maketitle

\begin{abstract}
Highly overparameterized models often predict well despite interpolating training data in complex domains, challenging the classical bias--variance tradeoff. We investigate whether this ``benign overfitting'' phenomenon extends to equity return prediction. Consistent with recent statistical theory, we document two key phenomena: first, a double descent pattern in the ridgeless model's prediction risk; and second, that while the optimal ridge model consistently outperforms its ridgeless counterpart, this performance gap becomes negligible at large parameter-to-observation ratios. Ultimately, however, both models fail to outperform a simple historical average. This empirical evidence aligns with our asymptotic results under the null hypothesis of zero slope coefficients, suggesting that standard equity predictors lack true forecasting power---even within highly flexible, nonlinear machine learning architectures. These findings reconcile modern and classical machine learning in asset pricing: in the absence of a true signal, they asymptotically collapse to the historical average benchmark.

\end{abstract}

\section{Introduction}\label{sec:Intro}

Modern machine learning models often achieve strong out-of-sample predictive performance despite operating in highly overparameterized regimes. This phenomenon, known as benign overfitting, occurs when models generalize well even as they nearly interpolate the training data \citep{belkin2019reconciling}. Benign overfitting appears to contradict the classical bias--variance tradeoff, which anticipates a U-shaped risk curve as model flexibility increases \citep{Hastie2009}. However, recent theoretical advancements \citep{bartlett2020benign, hastie2022surprises, mei2022generalization, tsigler2023benign} reveal that under appropriate spectral conditions and signal-to-noise alignments, interpolating estimators avoid catastrophic out-of-sample failure. Specifically, their prediction risk actually declines as the degree of overparameterization increases, providing a theoretical explanation for the double-descent risk curve documented by \citet{belkin2019reconciling}. 

We investigate benign overfitting within the context of market return prediction, a central problem in asset pricing.
Although attempts to forecast market returns date back more than a century \citep{Dow1920}, robust empirical evidence of return predictability remains elusive. Comprehensive studies consistently find that most standard predictors exhibit little out-of-sample forecasting ability \citep{Welch2008Comprehensive, Goyal2024}. However, modern machine learning methods have recently revitalized this literature, suggesting that economically meaningful signals may be recoverable through highly flexible nonlinear models, even when traditional linear approaches fail. 

Specifically, \citeauthor{Kelly2024} (\citeyear{Kelly2024}; hereafter KMZ) demonstrate that random Fourier features (RFFs)---generated from the predictor variables of \citet{Welch2008Comprehensive} via the approach of \citet{Rahimi2007}---deliver significant market-timing ability in ridge regressions. Their central conclusion is a ``virtue of complexity,'' whereby out-of-sample performance increases monotonically with the parameter-to-observation ratio, $\gamma=P/T$, as both $P$ and $T$ diverge. Importantly, however, equating complexity with the raw number of predictors 
$P$ can be misleading in regularized models. For ridge regression, the effective degrees of freedom (EDF) provide a more relevant measure of complexity, because they incorporate both model dimension and shrinkage; as a result, even a large $P$-feature specification may be effectively close to an intercept-only model under strong regularization.

Overall, this setting provides an ideal testbed for examining benign overfitting because it combines many of the challenges that characterize traditional equity-premium forecasting: weak signals, persistent and highly correlated predictors, potential nonlinearities, and severe overparameterization. Consequently, determining the precise sources of these documented gains---whether they arise from genuine nonlinear predictability or from the mechanical shrinkage properties of overparameterized models---has important implications for both machine learning theory and empirical asset pricing.

Figure~\ref{fig:double_descent} illustrates two key findings regarding benign overfitting in equity return prediction.\footnote{In Figure~\ref{fig:double_descent}, we set $P$ ranging from 2 to 12,000 for each $T$. The entire procedure is repeated 100 times with independent generations of RFF features. The out-of-sample test risk is calculated upon the average return across these 100 independent generations. For visual clarity, the figure displays the curve only up to $\gamma=5,$ beyond this range, the empirical curves continue to decline and plateau.} First, the ridgeless RFF estimator exhibits the familiar double-descent phenomenon, where test risk spikes near the interpolation threshold ($\gamma = P/T = 1$) before declining as $\gamma$ increases. Second, a heavily regularized ridge regression ($z = 10^{9}$) achieves a strictly lower risk level across all $\gamma$, empirically tracking the historical-average forecast. This empirical dominance is theoretically grounded \citep{hastie2022surprises}: under the null hypothesis of zero true slope coefficients, the prediction risk of the overparameterized ridgeless model decreases monotonically with $\gamma$ toward the irreducible noise floor of the true model. In addition, our framework dictates that the optimal ridge penalty diverges ($z^{\ast} = \infty$). Consequently, the optimally tuned model bypasses the double-descent curve entirely, collapsing directly to the historical average.

\begin{figure}[h]
\centering 
\includegraphics[width=0.8\textwidth]{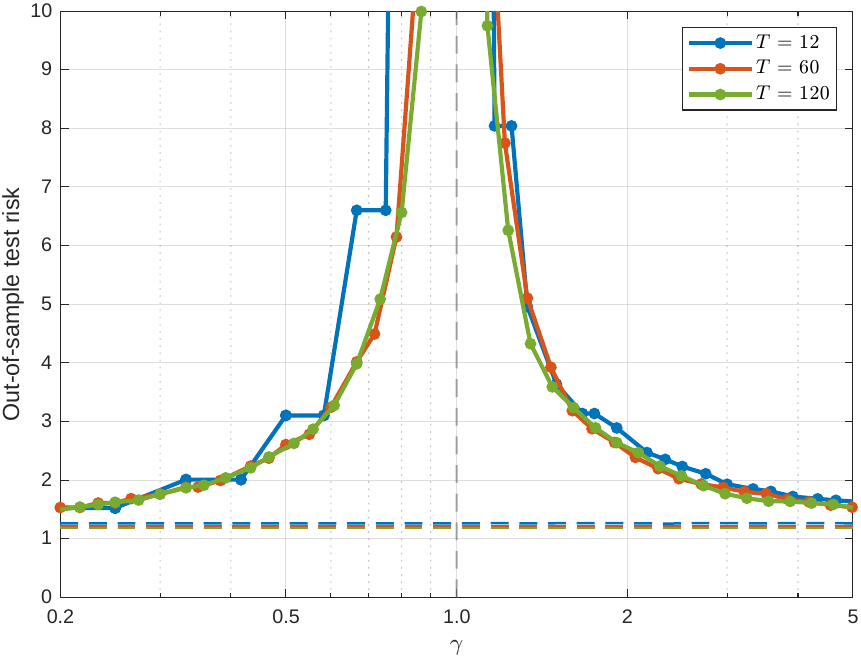}
\caption{Out-of-sample test risk (mean squared error) against $\gamma = P/T$. Each color corresponds to a different training window length ($T = 12$, $60$, and $120$ months). Solid lines show the ridgeless estimator ($z \to 0^+$); dashed lines show the heavily regularized benchmark ($z = 10^9$), which virtually overlay their historical averages respectively. The vertical dashed line at $\gamma = 1$ marks the interpolation threshold.}
\label{fig:double_descent}
\end{figure}

These findings reveal a  crucial limitation of complex machine learning models in asset pricing: benign overfitting does not guarantee an enhanced capacity to forecast equity market returns. This reality is easily obscured because the declining risk in the highly overparameterized regime fosters the impression that larger models are inherently better at extracting true signals \citep{Kelly2024}. However, the fact that a simple historical average outperforms the complex ridgeless estimator explicitly refutes this virtue of complexity interpretation. Ultimately, no model can identify a signal that does not exist in the data; when the underlying features lack predictive power, the optimally regularized machine learning architecture simply collapses back to the historical average benchmark.

To formally characterize benign overfitting in market return prediction, this paper offers several distinct methodological and empirical contributions. Theoretically, we prove that under the null hypothesis of zero slope coefficients, the optimal ridge penalty diverges ($z \to \infty$), mathematically collapsing the high-dimensional RFF model to an intercept-only specification. This formally establishes the standard historical average as the theoretically justified benchmark for evaluating overparameterized prediction models. Furthermore, 
we establish that a cross-validation (CV) tuned ridge shrinkage parameter achieves the same asymptotic prediction risk as the oracle-optimal penalty -- a general result that holds regardless of whether the null is true. Empirically, we demonstrate that
this convergence to an intercept-only model holds across a wide range of z values. Consequently, when an intercept is included, RFF forecasts mechanically track the historical average across various penalties and training windows, explaining the model’s apparent pre-1975 success and subsequent post-1975 collapse. Finally, a block bootstrap analysis corroborates these findings, confirming that models trained on actual economic predictors do not significantly outperform those trained on randomly resampled covariates stripped of predictive content.

Our results also show that the market-timing ability of KMZ's no-intercept RFF model is attributable to its correlation with this historical average, driven by their use of non-centered features. When predictors are appropriately mean-centered, the RFF model's forecasting power vanishes, consistent with the null of zero slopes. Ultimately, we demonstrate that this reliance on the historical average benchmark extends beyond the RFF setting, explaining the performance of various momentum strategies examined by \citet{Nagel2025} and \citet{Kelly2025}.

The remainder of the paper is organized as follows. Section \ref{ModTheory} discusses model and theoretical results for the RFF model with optimal CV ridge regression and under null of zero slope. Section \ref{Model} illustrates data structure and model setting. Section \ref{Empirics} presents the main empirical findings. Section \ref{MomStra} compares RFF model forecasts with momentum-based strategies. Section \ref{Boot} presents the bootstrap simulation results. Section \ref{windows} extends to different training windows. Section \ref{rem:KMZnointercept} Reconciles with KMZ's main results. Section \ref{Conclusion} concludes the paper. Additional proofs and empirical results are provided in the Supplementary Materials. The computation code is available at : \url{https://github.com/Jiawei98/BenignOverfitting/tree/main}.

\section{Model, Estimation, and Theoretical Properties}\label{ModTheory}
\subsection{Model}

Using a low-dimensional vector of commonly used predictor variables $\boldsymbol{G}_t \in \mathbb{R}^d$, we generate a high-dimensional set of random Fourier features $\boldsymbol{S}_t \in \mathbb{R}^P$ using the method of \citet{Rahimi2007}:
\begin{equation}
\boldsymbol{S}_t = \sqrt{{\frac2 P}} \left[ \sin(\psi\boldsymbol{w}^\prime \boldsymbol{G}_t), \cos(\psi\boldsymbol{w}^\prime \boldsymbol{G}_t) \right]^\prime,
\label{rffeq}
\end{equation}
where $\{ w_j \}_{j=1}^{P/2} \subset \mathbb{R}^d$ are independent and identically distributed (i.i.d.) from standard normal distribution $N(0,\mathrm{I}).$ We then use $\boldsymbol{S}_t$ in a linear model to predict the excess stock market return, $R_{t+1}$:
\begin{equation}
R_{t+1} = \beta_0 + \boldsymbol{S}_t^\prime \boldsymbol{\beta} + \varepsilon_{t+1}, \quad \mathbb{E}[\varepsilon_{t+1}] = 0,
\label{modelintercept}
\end{equation}
where $\beta_0$ is the intercept, $\boldsymbol{\beta} = (\beta_1, \beta_2, \ldots, \beta_P)^\prime \in \mathbb{R}^P,$ is the $P$-dimensional vector of slope parameters, and $\varepsilon_{t+1}$ the i.i.d. regression error with $\mathbb{E}[\varepsilon^2_{t+1}]=\sigma^2$. 

Or equivalently, a population mean-centered model can be written as
\begin{equation}
R_{t+1} - \mathbb{E} (R_{t+1})  = (\boldsymbol{S}_t -  \mathbb{E} \boldsymbol{S}_t)^\prime \boldsymbol{\beta} + \varepsilon_{t+1}, \mathbb{E}[\varepsilon_{t+1}] = 0.
\label{modelpc}
\end{equation}
And 
\begin{equation}
\beta_0 =  \mathbb{E} (R_{t+1}) - (\mathbb{E} \boldsymbol{S}_t)^\prime \boldsymbol{\beta}.
\label{interceptp}
\end{equation}

Models (e.g. \cite{Kelly2024}, \cite{hastie2022surprises}) are often set up as 
\begin{equation}
R_{pc,t+1} = \boldsymbol{S}_{pc,t}^\prime \boldsymbol{\beta} + \varepsilon_{t+1}, \quad \mathbb{E}[\varepsilon_{t+1}] = 0,
\label{modelpc0}
\end{equation}
where population-centered
$R_{pc,t+1} = R_{t+1} - \mathbb{E} (R_{t+1})$ and $\boldsymbol{S}_{pc,t} = \boldsymbol{S}_t -  \mathbb{E} \boldsymbol{S}_t$ partly for notational simplicity. And the theories are developed accordingly.
For example, \cite{Kelly2024} assume $\mathbb{E} (\boldsymbol{S}_t) = 0$ (their Assumption 2), $\mathbb{E}[\varepsilon_{t+1}] = 0$  (their Assumption 1), $\beta_0=0,$ and hence $ \mathbb{E} (R_{t+1}) = 0$. For theory and assumptions, we will also follow this convention without loss of generality and we will make model assumptions as in \cite{Kelly2024}. We will focus on the empirical implementations, results, and their connection with theory primarily.\footnote{We intentionally set up our model and model assumptions such as i.i.d. errors closely following the latest literature \cite{Kelly2024}, \cite{hastie2022surprises} as much as possible so that we can understand benign overfitting and detect the misunderstanding and the apparent illusion of Virtue of Complexity. The theory is developed under these model assumptions. In reality, the market return is a financial time series. The predictors are persistent. Returns are heteroskedastic. The rolling window size is not large. The related theory under general model assumption is a challenging open research area to explore in the future.}

\subsection{Estimation}

Empirically, the standard ridge formulation solves:
\begin{equation}
\begin{aligned}
(\hat{\beta}_0, \hat{\boldsymbol{\beta}}_z) = \arg\min_{\beta_0, \boldsymbol{\beta}} &\left\{ \frac{1}{T} \sum_{t=0}^{T-1} \left( R_{t+1} - \beta_0 - \boldsymbol{S}_t^\prime \boldsymbol{\beta} \right)^2 \right.\\
&\left. \qquad + z \sum_{j=1}^{P} \beta_j^2 \right\}.
\end{aligned}
\end{equation}
The regularization shrinkage is on $\sum_{j=1}^{P} \beta_j^2,$ and does not shrink the intercept $\beta_0$.\footnote{See \cite{Hastie2009}. To make it easy for comparison, our notation $z$ here is the same as that in KMZ, while the penalty parameter in \cite{Hastie2009} is $\lambda=z*T$. Note KMZ does NOT include intercept in their empirical study.} We observe the training sample in $T$ period $\{ \boldsymbol{S}_t, R_{t+1} \}_{t=0}^{T-1}.$ 
In empirical applications, predictors and returns naturally exhibit non-zero means. We estimate a ridge regression that penalizes the slope coefficients but leaves the intercept unpenalized. It is a standard algebraic result \citep[e.g.,][]{Hastie2009} that estimating a ridge regression with an unpenalized intercept on original data is exactly equivalent to first sample-centering the data, estimating a no-intercept ridge regression on the centered variables, and then recovering the intercept afterwards.

Formally, denote $\mathbf{R} = [R_1, \ldots, R_T]^\prime \in \mathbb{R}^T$ and $\mathbf{S} = [\boldsymbol{S}_0^\prime, \ldots, \boldsymbol{S}_{T-1}^\prime]^\prime \in \mathbb{R}^{T \times P}$. Define the sample means $\bar{R} = T^{-1} \sum_{t=1}^T R_{t}$ and $\bar{\boldsymbol{S}} \in \mathbb{R}^P,$ the column-wise sample mean of $\mathbf{S}$. For ease of exposition, we use the sample centered response $\boldsymbol{R}_c = \boldsymbol{R} - \overline{R}\mathbf{1}_T$ and the sample centered design matrix $\mathbf{S}_c=\boldsymbol{S} - \mathbf{1}_T \overline{\boldsymbol{S}}^\prime$, where $\mathbf{1_T}$ is a $T\times 1$ vector of ones. Then the estimation is equivalent to solving ridge regression on the mean-centered variables:
\begin{equation}
\hat{\boldsymbol{\beta}}_z = \arg\min_{\boldsymbol{\beta}} \left\{ {\frac 1 T} \left\| \boldsymbol{R}_c - \mathbf{S}_c\boldsymbol{\beta} \right\|^2 + z \| \boldsymbol{\beta} \|^2 \right\},
\label{ridge}
\end{equation}
with intercept estimate recovered by:
\begin{equation}
\hat{\beta}_0 = \bar{R} - \bar{\mathbf{S}}^\prime \hat{\boldsymbol{\beta}}_z.
\label{intercept}
\end{equation}

Let $(\boldsymbol{s}_{\mathrm{new}}, R_{\mathrm{new}})$ denote a new test pair drawn independently from the same distribution as the training
observations. Its prediction is
$$\hat{R}_{\mathrm{new}} = \hat{\beta}_0 + \boldsymbol{s}_{\mathrm{new}}^\prime \hat{\boldsymbol{\beta}}_z  = \bar{R} + (\boldsymbol{s}_{\mathrm{new}} - \bar{\mathbf{S}})^\prime \hat{\boldsymbol{\beta}}_z = \bar{R} + \boldsymbol{s}_{c,\mathrm{new}} ^\prime \hat{\boldsymbol{\beta}}_z.$$
The ridge regression estimator for the slope coefficients $\boldsymbol{\beta}$ is:
\begin{equation}
 \hat{\boldsymbol{\beta}}_{z} = \left(\frac{1}{T}\mathbf{S}_c^{\prime}\mathbf{S}_c + z \mathbf{I} \right)^{-1}\left(\frac{1}{T}\mathbf{S}_c^{\prime}\boldsymbol{R}_c\right). 
 \label{eq:betahat}
 \end{equation}
 The ridgeless (min-norm) estimator is when $z\to 0^+,$ or equivalently minimizing the squared error loss while minimizing its $l_2$ norm.


\subsection{Effective Degrees of Freedom: A Measure of Model Complexity}
\label{sec:EDF}

We further introduce a model complexity measure, \emph{effective degrees of freedom (EDF)} (See \citet[Chapter 3]{Hastie2009}), which may offer a more informative measure of model complexity than simply counting the number of predictors $P$, especially in regularized settings such as ridge regression.

Specifically, even if the number of predictors $P$ is larger, it does not necessarily mean the model is more complex. For example, a model with $P=1000$ but with large shrinkage parameter $z=10^9$ reduces essentially to a constant intercept model. This is effectively a simpler model than a model with $P=15$ but with $z=0.$ The former has effectively degrees of freedom of 1, i.e. only one intercept parameter to estimate, while the latter has effectively degrees of freedom of 16. 

Following the general definition of effective degrees of freedom due to \citet{Efron2004}, one can view EDF as quantifying how strongly the fitted values respond to the observed data, that is,
\[
\mathrm{EDF}
:= \frac{1}{\sigma^2} \sum_{t=0}^{T-1} \mathrm{Cov}(R_{t+1},\hat{R}_{t+1}),
\]
which reduces to the usual number of parameters in ordinary least squares (OLS).

In our ridge setting, let $\mathbf{S}_c$ denote the centered $T \times P$ design matrix and $\mathbf{R}_c$ the centered response vector.
For a given penalty $z>0$, the ridge fitted values satisfy
\begin{equation}
\begin{aligned}
\mathbf{S}_c \hat{\boldsymbol{\beta}}_z
&= \mathbf{H}_c(z)\,\mathbf{R}_c, \\
\quad\text{where}\quad
\mathbf{H}_c(z)
&= \frac{1}{T}\,\mathbf{S}_c
\Bigl(\tfrac{1}{T}\mathbf{S}_c^\prime\mathbf{S}_c + z\,\mathbf{I}_P\Bigr)^{-1}
\mathbf{S}_c^\prime.
\end{aligned}
\end{equation}
The matrix $\mathbf{H}_c(z)$ plays the role of a \emph{hat matrix} or \emph{smoother matrix}: it maps the observed outcomes to their fitted values.
When $z=0$ this reduces to the familiar OLS hat matrix, and in the nonparametric regression and generalized additive model literatures $\mathbf{H}_c(z)$ is often called the smoothing matrix \citep{HT1990}.

The effective degrees of freedom of the ridge estimator are equivalently defined as 
\begin{equation}
\label{eq:EDF-ridge-trace}
\mathrm{EDF}_{\text{ridge}}(z)
:= \operatorname{tr}\{\mathbf{H}_c(z)\} + 1,
\end{equation}
where the additional $1$ accounts for the intercept.
If $\tfrac{1}{T}\mathbf{S}_c^\prime\mathbf{S}_c$ has eigenvalues $d^2_1,\dots,d^2_P$, then \eqref{eq:EDF-ridge-trace} has the closed-form expression
$
\label{eq:EDF-ridge-eigs}
\mathrm{EDF}_{\text{ridge}}(z)
= \sum_{j=1}^{P} \frac{d^2_j}{d^2_j + z} + 1.
$
As $z \to 0$, we have $\mathrm{EDF}_{\text{ridge}}(z)$ goes to $P+1$, recovering the usual notion of degrees of freedom (the number of free parameters) in the OLS model.
As $z \to \infty$, we obtain $\mathrm{EDF}_{\text{ridge}}(z)$ goes to 1, corresponding to a pure intercept, i.e., the historical mean.
Thus, given $P$, $\mathrm{EDF}_{\text{ridge}}(z)$ is a decreasing function of $z$, reflecting how stronger regularization reduces the effective model complexity.

Overall, EDF combines both the nominal dimension $P$ and the amount of shrinkage $z$ into a single, interpretable complexity measure.
In high-dimensional financial applications with $P > T$, ridge regression can make heavily overparameterized models behave effectively as low-dimensional ones, and the effective degrees of freedom provides a convenient scalar summary of this shrinkage-induced reduction in model complexity.
A larger EDF indicates a more flexible, higher-variance fit, whereas a smaller EDF corresponds to a smoother, effectively simpler model.

\subsection{Theoretical Results under the Null of Zero Slope}

We are interested in the behavior of $\hat{\boldsymbol{\beta}}_{z}$ under the null hypothesis $\boldsymbol{\beta} = \boldsymbol{0}$, which states that none of the RFF terms have predictive power.  
Under the null hypothesis of zero slope, model~(\ref{modelintercept}) reduces to an intercept-only specification:
\begin{equation}
R_{t+1} = \beta_0 + \varepsilon_{t+1}, \quad \mathbb{E}[\varepsilon_{t+1}] = 0.
\label{null}
\end{equation}
Equation~(\ref{null}) implies that the historical average---that is, the mean return in the training sample---is the optimal forecast of the return in the next period.

Define the out-of-sample prediction error
\footnote{The prediction error defined in Page 958 \cite{hastie2022surprises} should be more precisely called as model error, where 
$$
\mathrm{ME}(z) \;=\; \mathbb{E}
\!\Bigl[\bigl((\hat\beta_0 + \boldsymbol{s}_{\mathrm{new}}^\prime
\hat{\boldsymbol{\beta}}_z)-(\beta_0 + \boldsymbol{s}_{\mathrm{new}}^\prime
{\boldsymbol{\beta}})\bigr)^2
\,\Big|\, \mathbf{S}\Bigr],
$$
which is equivalent to PE(z) minus a constant, the irreducible model error $\sigma^2$ (see Page 52 of \cite{Hastie2009}.) That is, $\mathrm{PE}(z)= ME(z) + \sigma^2.$ 
}
conditional on the training sample (see e.g., \cite{Hastie2009}, \citet{patil2021uniform}) as
\begin{equation}
\resizebox{0.5\linewidth}{!}{
$\mathrm{PE}(z) \;=\; \mathbb{E}_{(\boldsymbol{s}_{\mathrm{new}},\,R_{\mathrm{new}})}
\!\Bigl[\bigl(R_{\mathrm{new}} - (\hat\beta_0 + \boldsymbol{s}_{\mathrm{new}}^\prime
\hat{\boldsymbol{\beta}}_z)\bigr)^2
\,\Big|\, \mathbf{S},\,\mathbf{R}\Bigr].$}
\label{eq:PE}
\end{equation}

We derive the asymptotic results under the joint limit $T \to \infty$, $P \to \infty$ with $P/T \to \gamma \in (0, \infty)$.  
The assumption and proof that rely on a corollary and theorem from \citet{hastie2022surprises} involving random matrix theory are provided in Appendix~\ref{AppendixA1}. 

\begin{theorem} [High-dimensional $P$ Asymptotics] \label{theory1}
Under the null hypothesis of zero slope, the ridge regression optimal $z^{*}=\infty$, and under the optimal $z^{*}$, the prediction error (PE)
$$
PE(z^{*}) \rightarrow \sigma^2 
\quad \text{as} \quad P,T \to \infty \quad \text{with} \quad P/T \to \gamma \in (0,\infty).
$$
\end{theorem}

\subsection{Ridge Regression Penalty Parameter and Cross-Validation}
\label{sec:ridge_cv}

For ridge regression, the shrinkage penalty parameter $z$ is very important. Theoretically, the optimal tuning shrinkage parameter is 
$
z^* = \arg\min_{z \in
\mathcal{Z}}\,\mathrm{PE}(z) 
$
over a candidate set $\mathcal{Z},$ 
which depends on the unknown data–generating mechanism, including the signal–to–noise ratio and the feature covariance structure. 

Empirically, $\mathrm{PE}(z)$ cannot be computed directly because the true
data-generating process is unknown. Cross-validation (CV) \citep{Hastie2009} is commonly adopted to choose the data-driven shrinkage penalty parameter by using the observed data itself to estimate $\mathrm{PE}(z)$ for each candidate $z$, and then selecting the value that minimizes this data-driven estimate.  

Cross-validation (CV) \citep{Hastie2009} is commonly adopted to choose the data-driven shrinkage penalty parameter by using the observed data itself to estimate $\mathrm{PE}(z)$ for each candidate $z$, and then selecting the value that minimizes this data-driven estimate. The most
direct variant is leave-one-out cross-validation: for each
observation $i$, the model is refit on the remaining observations, and the prediction error on the held-out observation is recorded. Averaging these held-out errors gives an estimate of how well the model predicts out-of-sample. 

For each $t = 0,\ldots,T-1$, let $\hat{\boldsymbol{\beta}}_z^{(-i)}$ denote the
ridge estimator trained on all observations except the $i$-th. The CV error
is defined as:
\begin{equation}
\mathrm{CV}(z) \;=\; \frac{1}{T}\sum_{t}
\Bigl(R_{c,t+1} - S_{c,t}^\prime \hat{\boldsymbol{\beta}}_z^{(-i)}\Bigr)^2,
\label{eq:loocv_def}
\end{equation}
where the ridge slope $\hat{\boldsymbol{\beta}}_z^{(-i)}$
is trained on all but left-out data point.

We provide the theory that the CV-tuned ridge estimator is asymptotically optimal closely following \cite{patil2021uniform} with relaxed assumptions and stronger uniform consistency results than \cite{hastie2022surprises}.
\begin{theorem}[CV Uniform Consistency, \cite{patil2021uniform}]
\label{thm:cv}
Under Assumptions A1 to A4 in Appendix~\ref{AppendixA2}, let $P, T \to \infty$
with $P/T \to \gamma \in (0,\infty)$. For every compact interval $\mathcal{Z}
\subseteq (z_{\min}, \infty)$,
\begin{equation}
\sup_{z \in \mathcal{Z}}\,\bigl|\mathrm{CV}(z) - \mathrm{PE}(z)\bigr|
\;\xrightarrow{a.s.}\; 0.
\label{eq:gcv_unif}
\end{equation}
Consequently, the CV-tuned estimator achieves the same asymptotic prediction error as the optimally-tuned estimator:
\begin{equation}
\mathrm{PE}(\hat{z}_{\mathrm{CV}}) - \mathrm{PE}(z^*)
\;\xrightarrow{a.s.}\; 0.
\label{eq:cv_optimal}
\end{equation}
\end{theorem}

We note that the familiar double-descent phenomenon emerges naturally even under the null hypothesis of zero slope coefficients ($\boldsymbol{\beta} = \mathbf{0}$). To understand the statistical mechanics behind this, it is instructive to evaluate the out-of-sample prediction risk through its fundamental decomposition: the sum of irreducible noise ($\sigma^2$), bias risk, and variance risk. Following Theorem 1 of \cite{hastie2022surprises}, under isotrpoic features and high-dimensional asymptotics for a fixed parameter-to-observation ratio $\gamma = P/T > 1$, the prediction risk of the ridgeless estimator decomposes into an explicit bias component, $\|\boldsymbol{\beta}\|^2 \frac{\gamma}{\gamma - 1}$, and a variance component, $\sigma^2 \frac{1}{\gamma - 1}$. Under the null hypothesis, the lack of a true predictive signal renders the bias risk exactly zero. The out-of-sample error is therefore entirely driven by the variance component, $\sigma^2 \frac{1}{\gamma - 1}$. Hence, near the interpolation threshold ($\gamma \to 1$), the variance risk explodes, producing the characteristic spike in error; yet, as $\gamma$ grows increasingly large, the variance risk monotonically vanishes, generating the second ``descent.'' Because the irreducible noise represents the minimum prediction risk, eliminating both bias and variance in this highly overparameterized limit forces the total prediction risk to approach the global minimum, $\sigma^2$. Economically, this implies that the high-dimensional architecture simply collapses to the optimal forecast: the historical average.

\section{Empirical Results}

\subsection{Model Specification and Evaluation Metrics}\label{Model}

We evaluate an RFF model constructed from the 14 predictor variables in \citet{Welch2008Comprehensive}, which includes \texttt{dfy}, \texttt{infl}, \texttt{svar}, \texttt{de}, \texttt{lty}, \texttt{tms}, \texttt{tbl}, \texttt{dfr}, \texttt{dp}, \texttt{dy}, \texttt{ltr}, \texttt{ep}, \texttt{b/m}, \texttt{ntis}, plus the lagged excess equity market return, in total 15 variables. More detail about variables is in Appendix~\ref{AppendixB}. And we use the monthly data from 1927 to 2020 as the full sample.

With the 15 predictor variables, the RFF features are constructed from Equation~(\ref{rffeq}), where $G_t$ is the $15 \times 1$ predictor vector, $w_j$ is a $15 \times 1$ random weight vector drawn i.i.d.\ from a standard normal distribution, and $\psi = 2$ is the bandwidth parameter controlling the Gaussian kernel bandwidth. Since the RFF weight matrix $\boldsymbol{w}$ is drawn randomly, out-of-sample performance estimates can be noisy. To address this issue, the entire procedure is repeated 100 times with independent draws of $\boldsymbol{w}$, each repetition using a distinct random seed, we then calculate the average return across these 100 independent draws as the portfolio return. All reported statistics---including predicted returns, managed portfolio returns, out-of-sample $R^2$, Sharpe ratios, and alphas---are based on the averaged return. The number of RFF features is $P = 12{,}000$ per repetition following KMZ, e.g., $\gamma=P/T=1000$ in 12-month training window.

Following KMZ, before generating the RFF features, returns are volatility standardized by their trailing 12-month return standard deviation, which is computed from the uncentered second moment. Predictors are standardized using an expanding window historical standard deviation. At least 36 months of data are required for the initial standardization, so the usable sample begins in 1930. To perform regression estimation at each time $t$, the training-sample RFF features and the out-of-sample RFF vector are further standardized by their standard deviations and mean computed within the training window. 

We make out-of-sample forecasts using a rolling window. At each month $t$, the model is estimated using the most recent $T=12$ months of data to generate a one-step-ahead predicted excess market return, $\hat{R}_{t+1}$. For brevity, the results for $T \in \{60, 120\}$ are mainly reported in the Appendix \ref{AppendixC}. We follow \citet{Welch2008Comprehensive} in dividing the sample into pre- and post-1975 subsamples, in addition to the full sample. For the early sample, the testing windows are 1931--1974, 1935--1974, and 1940--1974 for training sample lengths of 12, 60, and 120 months, respectively. For the late sample, the testing window is 1975--2020 for all training sample lengths. For the full sample (1926--2020), the testing windows are 1931--2020, 1935--2020, and 1940--2020 for training sample lengths of 12, 60, and 120 months, respectively. 

We report results for four regularization specifications: the ridgeless case ($z = 0$), which imposes minimal shrinkage; the heavily penalized case ($z = 10^9$), which collapses the model toward the historical average; and $z = 10^3$, which achieves the highest Sharpe ratio across the wide range of regularization parameters considered in \citet{Kelly2024}. We additionally report results for the cross-validated specification, where the ridge penalty is selected via rolling cross-validation at each forecast origin. For comparison, we consider a standard linear model incorporating the 15 original predictors, an intercept-only specification serving as the historical-average benchmark, and the buy-and-hold market portfolio return.

For each model, the foretasted market excess return $\hat{R}_{t+1}$ scales the position in the market portfolio, yielding a managed portfolio excess return of $R^p_{t+1} = \hat{R}_{t+1} R_{t+1}$. As $R_{t+1}$ is volatility-standardized using prior data, it has a higher mean than the raw return because the return volatility is always less than one. We evaluate the performance of the model through four primary metrics: 

\begin{itemize}
    \item[1.] The out-of-sample $R^2$ \citep{campbell2008predicting},
    $$R^2_{\text{OOS}} = 1 - \frac{\sum_t (R_{t+1} - \hat{R}_{t+1})^2}{\sum_t (R_{t+1} - \bar{R}_{t+1})^2},$$
    quantifies the reduction in mean squared prediction error relative to the historical-average benchmark $\bar{R}_{t+1}$, which is calculated by averaging the return from $R_{t-T+1}$ to $R_{t}$. By construction, $R^2_{\text{OOS}}$ is positive if the model outperforms the benchmark and zero if the forecast converges to $\bar{R}_{t+1}$.

    \item[2.] The Sharpe ratio \citep{Sharpe1966, Sharpe1994},
    $$\text{SR} = \frac{\mathbb{E}[R^{p}_{t+1}]}{\sqrt{\mathrm{Var}(R^{p}_{t+1})}},$$
    is a standard measure of risk-adjusted economic performance.

    \item[3.] The correlation of managed portfolio return with the historical-average benchmark return, which measures how closely the managed portfolio co-moves with the historical-average portfolio.

    \item[4.] Two risk-adjusted alpha measures. $\alpha_1$ is the standard CAPM alpha, obtained by regressing the managed portfolio return on an intercept and the market return. $\alpha_2$ is obtained by regressing the managed portfolio return on an intercept, the market return, and historical-average return. For each alpha measure, we report its associated $t$-statistic.
\end{itemize}

\subsection{RFF Model versus Intercept-Only Model}\label{Empirics}

\begin{table}
\TABLE
{Market-Timing Performance: 12-Month Training Windows\label{tab:12month_s1}}
{%
\begin{tabular}{{l L C R R R R R}}
\toprule
& Portfolio & $k$ & $\text{OOS } R^2$ & SR & Corr & $t_{\alpha_1}$ & $t_{\alpha_2}$ \\
& & $(z = 10^k)$ & & & & & \\
\midrule
\multicolumn{8}{l}{\textbf{Panel A: 1931--1974}} \\
\midrule
1  & Market   &           &         & 0.46 &      &       &        \\
2  & History  &           &         & 0.52 &      & 3.00  &        \\
3  & RFF      & 9         &  0.00   & 0.52 & 1.00 & 3.00  &  1.05  \\
4  & Linear   & 9         &  0.00   & 0.52 & 1.00 & 3.00  & -0.24  \\
5  & RFF      & 3         &  0.00   & 0.55 & 0.96 & 3.32  &  1.32  \\
6  & Linear   & 3         &  0.00   & 0.52 & 1.00 & 3.00  & -0.25  \\
7  & RFF      & Ridgeless & -0.02   & 0.56 & 0.82 & 3.59  &  1.49  \\
8  & RFF      & CV        & -0.01   & 0.53 & 0.90 & 3.21  &  0.82  \\
\midrule
\multicolumn{8}{l}{\textbf{Panel B: 1975--2020}} \\
\midrule
9  & Market   &           &         & 0.56 &      &       &        \\
10 & History  &           &         & 0.35 &      & 0.93  &        \\
11 & RFF      & 9         &  0.00   & 0.35 & 1.00 & 0.93  & -0.02  \\
12 & Linear   & 9         & -0.00   & 0.35 & 1.00 & 0.93  & -1.21  \\
13 & RFF      & 3         & -0.00   & 0.35 & 0.98 & 0.92  &  0.04  \\
14 & Linear   & 3         & -0.00   & 0.35 & 1.00 & 0.92  & -1.22  \\
15 & RFF      & Ridgeless & -0.01   & 0.34 & 0.91 & 0.97  &  0.24  \\
16 & RFF      & CV        & -0.01   & 0.34 & 0.96 & 0.86  & -0.18  \\
\midrule
\multicolumn{8}{l}{\textbf{Panel C: 1931--2020}} \\
\midrule
17 & Market   &           &         & 0.51 &      &       &        \\
18 & History  &           &         & 0.43 &      & 2.70  &        \\
19 & RFF      & 9         &  0.00   & 0.43 & 1.00 & 2.70  &  0.91  \\
20 & Linear   & 9         & -0.00   & 0.43 & 1.00 & 2.70  & -1.02  \\
21 & RFF      & 3         &  0.00   & 0.45 & 0.97 & 2.91  &  1.16  \\
22 & Linear   & 3         & -0.00   & 0.43 & 1.00 & 2.70  & -1.02  \\
23 & RFF      & Ridgeless & -0.01   & 0.45 & 0.86 & 3.18  &  1.46  \\
24 & RFF      & CV        & -0.01   & 0.43 & 0.93 & 2.82  &  0.69  \\
\bottomrule
\end{tabular}%
}{\textbf{Note:} Market denotes the equity market portfolio. History denotes the historical average benchmark. RFF denotes the nonlinear model with 12{,}000 random Fourier features. Linear denotes the linear model of the 15 original predictor variables. Both features and response are demeaned prior to estimation ($\mathrm{dX}=1$, $\mathrm{dY}=1$). $k$ denotes the regularization exponent ($z = 10^k$); Ridgeless denotes the ridgeless estimator ($z = 0$); CV denotes cross-validation selected $z$. $t$-statistics use Newey--West standard errors with four lags, approximating the optimal $\text{Lag} =  4 \left(T/100\right)^{2/9}$ in \cite{newey1994automatic}.}
\end{table}

Table~\ref{tab:12month_s1} reports the out-of-sample performance of models estimated using a 12-month training window under the highly overparameterized specification of $12{,}000$ random Fourier features, the most complex model examined by \citet{Kelly2024}. As illustrated in their Figure 8, this specific architecture yields the highest Sharpe ratio for a given level of regularization, serving as the basis for their virtue of complexity conclusion. Panels A, B, and C detail the results for the pre-1975, post-1975, and full sample periods, respectively.

Theorem~\ref{theory1} establishes that, under the null hypothesis of zero slope coefficients, the RFF model with an intercept mechanically converges to the historical average as the ridge penalty $z \to \infty$. Our results are consistent with this implication. At the extreme regulation level of $z = 10^9$, the RFF forecast achieves a correlation of 1.00 with the historical-average benchmark and an essentially zero $R^2_{\text{OOS}}$ across all three panels of Table~\ref{tab:12month_s1} (rows 3, 11, and 19). Furthermore, the two models yield nearly identical Sharpe ratios and CAPM $\alpha$ statistics, confirming that the heavily regularized RFF model contains no marginal predictive content beyond that of the unconditional mean.

The ridgeless ($z \to 0^+$) specification yields a negative $R^2_{\text{OOS}}$ across all panels ($-0.02$ in Panel A, $-0.01$ in Panel B, and $-0.01$ in Panel C), establishing that its point forecasts slightly underperform the historical average. Although the CAPM alpha appears statistically significant ($t_{\alpha_1} = 3.59$ in Panel A and $3.18$ in Panel C), such significance vanishes completely once we control for the historical-average managed portfolio ($t_{\alpha_2} = 1.49$, $0.24$, and $1.46$, respectively). The disappearance of the alpha reveals that any perceived market-timing ability is entirely subsumed by the historical-average benchmark. Even absent an explicit ridge penalty, the ridgeless specification generates managed portfolio returns that correlate heavily with that simple benchmark (0.82, 0.91, and 0.86). As we discuss in Section \ref{Model}, these findings reflect double descent. In a heavily overparameterized regime---where the feature space $P=12,000$ dwarfs the sample size of $T=12$---the ridgeless estimator effectively converges toward the true underlying model of zero slope coefficients. Ultimately, rather than demonstrating a virtue of complexity, our results suggest that the model's performance is driven purely by a mechanical convergence toward the historical average.

At regularization $z = 10^3$---KMZ's preferred specification---the RFF model yields a virtually zero $R^2_{\text{OOS}}$ across Panels A through C, alongside Sharpe ratios ($0.55$, $0.35$, and $0.45$) that are similar to, or slightly higher than, those of the ridgeless estimator. Furthermore, the managed portfolio returns maintain near-perfect correlation with the historical-average benchmark across all panels ($0.96$, $0.98$, and $0.97$). Consequently, the controlled alpha ($\alpha_2$) remains statistically insignificant in all periods ($t_{\alpha_2} = 1.32$, $0.04$, and $1.16$). These results are again entirely consistent with the theoretical dominance of optimally regularized ridge regression, which efficiently shrinks the estimator back toward the true underlying model of zero slope coefficients.

A fixed ridge penalty, such as $z=10^3$ in KMZ using the full data, is not consistent with a real-time out-of-sample forecasting design unless it is determined solely from the training data available at each forecast origin. Accordingly, we follow Theorem \ref{thm:cv} to estimate the optimal penalty parameter for ridge regression using cross-validation within each training window. At each month $t$, we select the regularization parameter $z$ through a one-step-ahead validation scheme: we use months $t-T+1$ through $t-1$ as the training set and month $t$ as the validation set, choosing the value of $z$ in the range $[10^{-9}, 10^{9}]$ that minimizes the validation error. The model is then refit on the full window from month $t-T+1$ through $t$ using the chosen $z$ to make the out-of-sample forecast $\hat{R}_{t+1}$. As with the fixed-$z$ specifications, this procedure is repeated across 100 independent sets of RFF weights, and the resulting forecasts are averaged to form the final prediction. As reported in rows 8, 16, and 24 of Table~\ref{tab:12month_s1}, the CV portfolio delivers an out-of-sample $R^2$ of $-0.01$ in both sub-periods and the full sample, with Sharpe ratios of 0.53, 0.34, and 0.43 the corresponding $t_{\alpha_2}$ statistics (0.82, -0.18, and 0.69) and correlation with history average (0.90, 0.96, and 0.93) confirm that the CV portfolio's return is statistically indistinguishable from the historical-average benchmark.

The validation-optimal regularization parameter $z$ varies across seeds due to different random weight matrix $\boldsymbol{w}$. To summarize the optial penalty, we calculate the median (mean) of the CV-selected $z$ across the 100 seeds at each forecast month $t+1$, and then take the time-series median of these monthly values. The resulting median-of-medians (median-of-means) for $z$ is $4.58\times 10^{3}$ ($4.70 \times 10^{8}$) for the 12-month window, with corresponding EDF medians of 2.49. Notably, cross-validation frequently selects the maximum regularization value of $10^{9}$ (in 47.68\% of months).\footnote{Results are similar for 60- and 120-month training windows. The median-of-medians (median-of-means) for $z$ is $3.59 \times 10^{2}$ ($4.40 \times 10^{8}$) for the 60-month window, and $1.87 \times 10^{2}$ ($4.30 \times 10^{8}$) for the 120-month window, with corresponding EDF medians of 20.91 and 40.84. Cross-validation selects the maximum regularization value of $10^{9}$ in 44.26\% and 41.79\% of months, respectively.} This extreme shrinkage pushes the slope coefficients toward zero, effectively collapsing the model onto the historical average. Ultimately, these results suggest that the RFF features yield negligible incremental forecasting power, supporting the zero-slope hypothesis of Theorem~\ref{theory1} that the model's predictions do not improve upon the historical average.

The sub-period results reinforce this interpretation and further explain the predictive power of the RFF models. In the early sample (1931--1974, Panel~A), the market portfolio attains an annualized Sharpe ratio of 0.46, while the historical average benchmark achieves 0.52 with a statistically significant CAPM $\alpha$ ($t_{\alpha_1} = 3.00$). This indicates that the average return over the prior 12 months possesses strong market-timing ability during this period.\footnote{Its correlation with the one-month-ahead equity market return is significant at the 1\% level (untabulated).} However, this significance disappears in the later sample (1975--2020, Panel~B). In contrast to the early period, all models exhibit a lower Sharpe ratio than the market portfolio, and none produce a statistically significant CAPM~$\alpha$ at the 10\% level. This breakdown in predictive power demonstrates that the prior 12-month average return is no longer informative in the more recent sample. When pooling the full 1931--2020 period (Panel~C), the deterioration in the second subsample drags down overall performance, resulting in the market portfolio achieving a higher Sharpe ratio than all other models.

Finally, Figure~\ref{fig:alpha2_tstatd} presents the Newey–West $t$-statistic of $\alpha_2$ for the 12-month RFF model across a wide range of regularization values $z$ for the pre-1975, post-1975, and full sample periods. All $t_{\alpha_2}$ are below the 10\% significance threshold, confirming that the absence of incremental predictability is robust to the choice of regularization parameter and sub-sample periods. Taken together, the evidence across all three sub-periods consistently indicates that the RFF model, whether ridgeless or regularized, does not enhance the performance beyond the historical-average managed benchmark.


\begin{figure}[h]
\centering 
\includegraphics[width=0.8\textwidth]{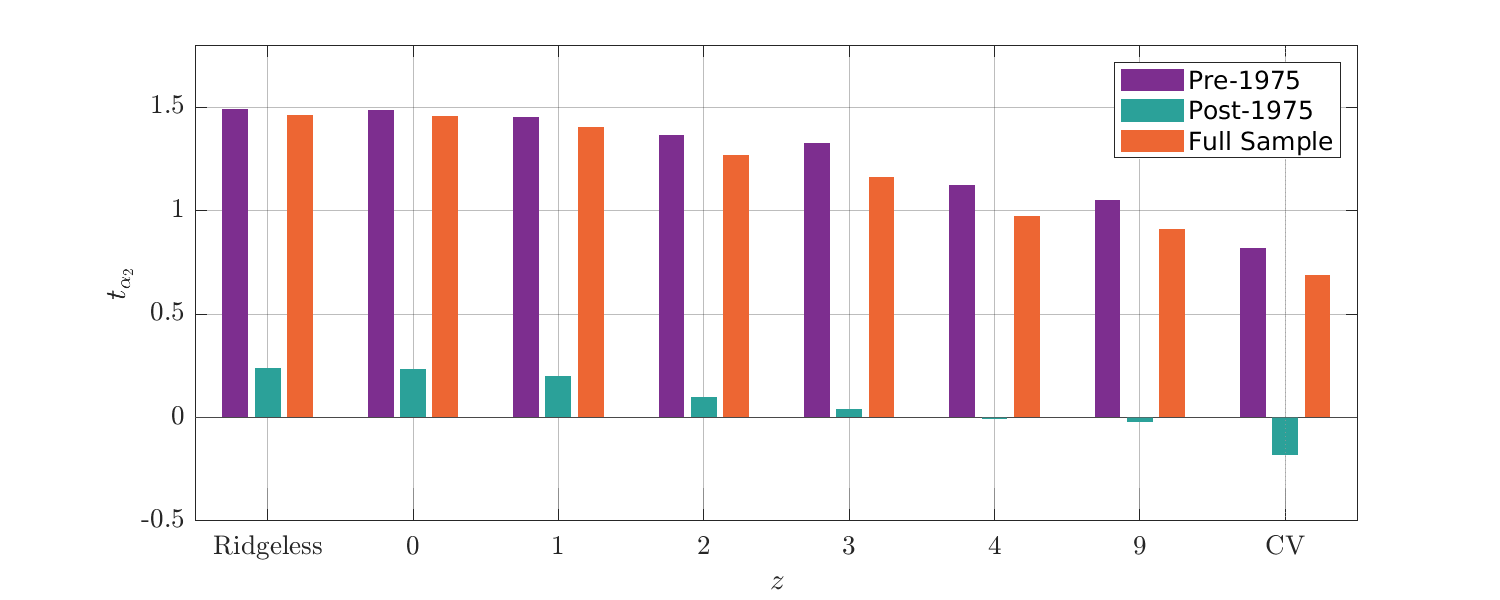}
\caption{Newey--West $t$-Statistic of $\alpha_2$ for the RFF Model across regularization values: 12-Month Training Windows}
\label{fig:alpha2_tstatd}
\end{figure}

\subsection{RFF Model and Momentum Strategies} \label{MomStra}
Our results show that RFF models have market timing ability similar to that of the historical average. The evidence is consistent with KMZ's conjecture that an RFF model estimated over a short 12-month training window may simply capture the positive serial correlation in stock market returns in the presence of highly persistent predictor variables. To better understand the underlying economic drivers of this overparameterized ridge regression, we compare its forecasts against a comprehensive suite of closely related momentum-based strategies. Our main empirical finding is that, after controlling for the historical-average benchmark, none of these strategies deliver significant abnormal performance across the major sample periods. This result indicates that, like the RFF model, their apparent predictive power merely reflects the same variation already embedded in the historical average.

We begin with the time-series momentum benchmark emphasized by KMZ:
\[
\hat{\mu}^{\text{TsMom}}_{t+1|t}
=
\frac{\frac{1}{12}\sum_{k=0}^{11} R_{t-k}}
{\sqrt{\frac{1}{12}\sum_{j=0}^{11}\left(R_{t-j}-\left(\frac{1}{12}\sum_{k=0}^{11}R_{t-k}\right)\right)^2}}.
\]
KMZ show that the abnormal return of the RFF model is not explained by this momentum model. Another important benchmark is the volatility-timed momentum strategy studied by \citet{Nagel2025}, who establishes an analytical link between the RFF model with a 12-month training window and the strategy
$
\hat{\mu}^{\text{VolMom}}_{t+1|t}
=
0.05 \frac{1}{\hat{\sigma}^2_{x,t}}
\sum_{k=0}^{11}\frac{12-k}{78}R_{t-k},
$
where $\hat{\sigma}^2_{x,t}$ denotes the average variance of the 15 original predictor variables in the 12-month training sample ending in month $t$. For robustness, we also consider four related benchmark strategies discussed by \citet{Kelly2025}:
\begin{align*}
\hat{\mu}^{\text{History}}_{t+1|t}=\frac{1}{12}\sum_{k=0}^{11}R_{t-k}, \quad
\hat{\mu}^{\text{DW}}_{t+1|t}=\sum_{k=0}^{11}\frac{12-k}{78}R_{t-k}, \\
\hat{\mu}^{\text{PV}}_{t+1|t}=\frac{1}{\hat{\sigma}^2_{x,t}}, \quad
\hat{\mu}^{\text{EWPV}}_{t+1|t}=\frac{1}{\hat{\sigma}^2_{x,t}}\sum_{k=0}^{11}R_{t-k}.
\end{align*}
These strategies isolate different components of the short-window forecasting rule: the historical average, duration-weighted momentum, predictor volatility, and a simple interaction between predictor volatility and the historical average.

\begin{table}
\TABLE
{Historical Average Benchmark and Momentum Strategies: 12-Month Training Windows\label{tab:mom12month}}
{\begin{tabular}{l L R R R R}
\toprule
& Portfolio & SR & Corr & $t_{\alpha_1}$ & $t_{\alpha_2}$ \\
\midrule
\multicolumn{6}{l}{\textbf{Panel A: 1931--1974}} \\
\midrule
1  & History & 0.52 &      &  2.99 &        \\
2  & TsMom   & 0.54 & 0.95 &  2.93 & -0.21  \\
3  & VolMom  & 0.52 & 0.64 &  2.49 &  0.64  \\
4  & DW      & 0.55 & 0.93 &  3.38 &  1.65  \\
5  & PV      & 0.54 & 0.41 &  2.32 &  0.96  \\
6  & EWPV    & 0.53 & 0.65 &  2.59 &  0.58  \\
\midrule
\multicolumn{6}{l}{\textbf{Panel B: 1975--2020}} \\
\midrule
7  & History & 0.34 &      &  0.85 &        \\
8  & TsMom   & 0.38 & 0.94 &  1.03 &  0.51  \\
9  & VolMom  & 0.19 & 0.77 &  0.00 & -1.20  \\
10 & DW      & 0.25 & 0.93 &  0.77 & -0.27  \\
11 & PV      & 0.50 & 0.49 &  0.22 & -0.31  \\
12 & EWPV    & 0.30 & 0.82 &  0.48 & -0.72  \\
\midrule
\multicolumn{6}{l}{\textbf{Panel C: 1931--2020}} \\
\midrule
13 & History & 0.43 &      &  2.70 &        \\
14 & TsMom   & 0.46 & 0.95 &  2.77 &  0.19  \\
15 & VolMom  & 0.38 & 0.66 &  2.17 &  0.26  \\
16 & DW      & 0.41 & 0.93 &  2.95 &  1.06  \\
17 & PV      & 0.50 & 0.42 &  2.10 &  0.96  \\
18 & EWPV    & 0.43 & 0.68 &  2.45 &  0.46  \\
\bottomrule
\end{tabular}}{\textbf{Note:} History denotes the historical average benchmark. TsMom denotes time-series momentum. VolMom denotes volatility-scaled momentum. DW denotes the Dynamically Weighted strategy. PV and EWPV denote the Parametric and Equal-Weighted Parametric Volatility strategies. Corr denotes the correlation of each portfolio return with the historical average. $t_{\alpha_1}$ and $t_{\alpha_2}$ denote $t$-statistics for the intercept from regressions on the market and on the market plus historical average, respectively, using Newey--West standard errors with four lags.}
\end{table}

Table~\ref{tab:mom12month} summarizes the main empirical result. Across three periods, these momentum strategies fail to generate significant abnormal returns once we control for its correlation with the historical average benchmark. Moreover, in the post-1975 sample, none delivers a statistically significant CAPM $\alpha$. Thus, while these strategies can provide useful economic descriptions of the short-window RFF forecast, their out-of-sample performance appears to be largely explained by the same benchmark variation captured by the historical average.

This perspective also helps reconcile subsequent interpretations of the RFF model. KMZ view time-series momentum as a natural benchmark motivated by persistent regressors, but find that it does not account for the model's market-timing performance. \citet{Nagel2025} instead emphasizes a mechanical link to volatility-timed momentum, while \citet{Kelly2025} argue that predictor-volatility-based strategies provide a more accurate economic distillation of the RFF model. Our results suggest a more direct interpretation: these alternative benchmarks remain closely tied to the historical average, which emerges naturally as the limiting forecast of ridge regression with an intercept under the null hypothesis of zero slope coefficients as $z\to\infty$.

\subsection{Predictor-Block Randomization Test} \label{Boot}

Our results show a close relationship between the RFF model and the historical average benchmark. To formally assess whether the economic predictors used in the RFF model has genuine predictive content, we implement a block-randomization procedure inspired by the block bootstrap of \citet{kunsch1989jackknife}.

At each forecast time $t+1$, we replace the RFF predictor with a randomly selected contiguous block of RFF predictors of the same length $T$, while keeping the realized return series fixed. This procedure preserves the serial dependence and marginal distribution of the predictor, as well as the realized time-series properties of returns, but breaks the temporal alignment between predictors and returns. The prediction on block bootstrap data therefore provides a null benchmark for model performance under the hypothesis that the RFF predictors contain no forecasting power for returns. We repeat the block bootstrap procedure $B=100$ times across all months. For each randomization, we re-fit the model following the same cross-validation procedure used in Section \ref{Empirics}, and compute the Sharpe ratio, out-of-sample $R^2$, $t_{\alpha_1}$, and $t_{\alpha_2}$. The resulting 100 statistics form the null distribution against which we test the zero-slope hypothesis.

\begin{table}
\TABLE
{Bootstrap Null Distribution of Out-of-Sample Performance: 12-Month Window\label{tab:bootstrap12}}
{%
\begin{tabular}{lRRRR}
\toprule
& SR & OOS $R^2$ & $t_{\alpha_1}$ & $t_{\alpha_2}$ \\
\midrule
Original Data     & 0.43 & $-$0.01 & 2.82 & 0.67 \\
Percentile        & 33.00 &   0.00 & 8.00 & 3.00 \\
Mean              & 0.43 &    0.00 & 2.88 & 1.33 \\
Median            & 0.43 &    0.00 & 2.88 & 1.36 \\
Std               & 0.01 &    0.00 & 0.05 & 0.36 \\
5th Pct           & 0.42 &    0.00 & 2.81 & 0.77 \\
95th Pct          & 0.44 &    0.00 & 2.97 & 1.95 \\
\midrule
Pooled Percentile & 11.00 &   0.00 & 69.33 & 1.00 \\
\bottomrule
\end{tabular}%
}{\textbf{Note:} Original Data denotes the out-of-sample performance of the model estimated on the original data. Percentile denotes the percentile of the original data performance within the window-specific bootstrap null distribution. Pooled Percentile reports, for $t_{\alpha_1}$ and $t_{\alpha_2}$, the percentile of the original data statistic within the bootstrap null distribution pooled across all training windows. Mean, Median, Std, 5th Pct, and 95th Pct summarise the distribution of out-of-sample performance across $B=100$ bootstrap replications, each using $100$ independently drawn sets of random Fourier features. SR denotes the annualized Sharpe ratio of the market timing strategy. OOS $R^2$ denotes the out-of-sample $R^2$ of the return forecast. $t_{\alpha_1}$ and $t_{\alpha_2}$ denote the Newey-West $t$-statistics from regressing the strategy return on the market ($\alpha_1$) and on the market plus the historical average ($\alpha_2$).}
\end{table}

Table~\ref{tab:bootstrap12} reports the bootstrap distribution of out-of-sample performance over the full sample for the 12-month training window. The original Sharpe ratio falls at the 33rd percentile of the bootstrap distribution, indicating that the model fails to deliver statistically significant market-timing performance relative to the bootstrap null. We observe a similar pattern for the out-of-sample $R^2$. In addition, while the resampled predictors consistently generate a positive and significant CAPM $\alpha_1$ (the fifth percentile of $t_{\alpha_1}$ equals 2.81), this abnormal return vanishes once we control for the historical average (the 95th percentile of $t_{\alpha_2}$ equals 1.95). These results replicate the close relationship between the RFF model and the historical average documented in the actual data. Moreover, the bootstrap distribution confirms that the original $\alpha_1$ is insignificant at conventional levels, demonstrating that the abnormal return of the RFF model simply reflects the historical average. Taken together, these findings provide little evidence that the model estimated on the original data outperforms the bootstrap benchmark. Instead, its observed timing performance is statistically indistinguishable from that generated by randomly resampled predictor blocks, which, by construction, contain no predictive information about future returns.

\subsection{Different Training Windows}\label{windows}
Although KMZ focus primarily on a 12-month training window, they also evaluate 60- and 120-month horizons. For comparison, Figures~\ref{fig:alpha2_tstatd_60} and \ref{fig:alpha2_tstatd_120} in the Appendix \ref{AppendixC} summarize our main results for these extended periods. The overall patterns largely mirror the 12-month setting in Figure~\ref{fig:alpha2_tstatd}, except that $\alpha_2$ reaches 5\% significance over the 120-month horizon when a low regularization penalty is applied. This significance, however, must be interpreted cautiously. 

In both the pre- and post-1975 subsamples, $\alpha_2$ fails to reach the 5\% significance threshold across all regularization values. Moreover, even if these estimates were statistically reliable, they would lack economic significance, as investors would have to endure extended wait times to realize any potential gains. Additionally, this apparent market-timing ability suffers from substantial look-ahead bias, as it implicitly assumes investors can identify a low regularization parameter ex ante. In practice, investors must estimate this optimal penalty. When we apply the cross-validated penalty value, $\alpha_2$ becomes statistically insignificant at the 5\% level across all three sample periods. Finally, standard critical values fail to account for the multiple testing of three window sizes across various penalties, exposing the initial finding to data-mining bias.

To address these biases, we employ the bootstrap methodology outlined in Section~\ref{Boot} and report the findings in Table~\ref{tab:bootstrap120}. We rely exclusively on the cross-validated tuning parameter to eliminate look-ahead bias. To correct for multiple testing across the 12-, 60-, and 120-month windows, we pool their empirical distributions to evaluate out-of-sample performance. While the actual data yields $t_{\alpha_2}=1.90$ (marginally significant at the 10\% level), this $t$-statistic ranks at only the 74th percentile of the simple 120-month bootstrap distribution and the 62nd percentile of the pooled distribution.

While KMZ emphasize a short training window to highlight the virtue of complexity, existing literature often employs longer samples to capture persistent variations in equity risk premia. To evaluate our results against this standard, we implement an expanding training window initialized in 1930. Our robustness check is motivated by \citet{Welch2008Comprehensive}, who demonstrate that out-of-sample predictive power of linear models deteriorates after 1974, likely because many established predictor variables suffer from data-mining bias: they often forecast equity market returns only within the specific samples used in the original studies that documented their significance. To mitigate this bias, we restrict our test sample to the post-1974 period (1975--2020). Untabulated results confirm that the RFF model possesses negligible market-timing ability in this period, generating a CAPM $\alpha$ that remains statistically insignificant at conventional levels.

Taken together, our evaluation across varying training sample sizes and rigorous out-of-sample tests confirms that the RFF model lacks robust, economically meaningful market-timing ability.

\subsection{Reconciling with KMZ's Empirical Specification and Related Literature}\label{rem:KMZnointercept}
KMZ exclude an intercept in their RFF ridge regression while employing non-centered market returns and non-centered features. This empirical specification is inconsistent with their theoretical population-centered model. Although \citet{Kelly2025} justify this specification by arguing that ``the monthly mean return is close to zero (0.007 in the KMZ sample),'' the average scaled market return used in the actual ridge regressions is $16\%$, with a $t$-statistic of 5. This indicates that the omitted intercept is both economically large and statistically significant. Consequently, the no-intercept specification is severely misspecified, especially under the null hypothesis of zero slope coefficients. Nevertheless, their main findings are strikingly similar to ours.

\begin{table}
\TABLE
{Market-Timing Performance: 12-Month Training Windows with KMZ No Intercept with Non-centered Market Return\label{tab:12month_center}}
{\begin{tabular}{l L C R R R R R R R R}
\toprule
\multicolumn{3}{c}{} & \multicolumn{4}{c}{\textbf{Non-centered Predictors (KMZ)}} & \multicolumn{4}{c}{\textbf{Centered Predictors}} \\
\cmidrule(lr){4-7} \cmidrule(lr){8-11}
& Portfolio & $k$ & SR & Corr & $t_{\alpha_1}$ & $t_{\alpha_2}$
                  & SR & Corr & $t_{\alpha_1}$ & $t_{\alpha_2}$ \\
& & $(z = 10^k)$ & & & & & & & & \\
\midrule
\multicolumn{11}{l}{\textbf{Panel A: 1931--1974}} \\
\midrule
1  & Market  &           &  0.46 &       &        &        &  0.46 &        &        &        \\
2  & History &           &  0.52 &       &  3.00  &        &  0.52 &        &  3.00  &        \\
3  & RFF     & 9         &  0.27 &  0.37 &  0.78  & -0.18  &  0.18 &  0.04  &  1.13  &  1.02  \\
4  & Linear  & 9         &  0.36 &  0.44 &  1.64  &  0.32  & -0.08 & -0.07  & -0.43  & -0.27  \\
5  & RFF     & 3         &  0.54 &  0.49 &  2.84  &  1.55  &  0.21 &  0.02  &  1.44  &  1.31  \\
6  & Linear  & 3         &  0.56 &  0.98 &  3.03  &  0.62  & -0.08 & -0.07  & -0.44  & -0.28  \\
7  & RFF     & Ridgeless &  0.46 &  0.36 &  2.81  &  1.62  &  0.26 &  0.08  &  1.84  &  1.48  \\
8  & RFF     & CV        &  0.41 &  0.28 &  1.83  &  1.13  &  0.17 &  0.05  &  1.02  &  0.82  \\
\midrule
\multicolumn{11}{l}{\textbf{Panel B: 1975--2020}} \\
\midrule
9  & Market  &           &  0.56 &       &        &        &  0.56 &        &        &        \\
10 & History &           &  0.35 &       &  0.93  &        &  0.35 &        &  0.93  &        \\
11 & RFF     & 9         &  0.39 &  0.58 &  1.40  &  0.85  & -0.02 & -0.05  & -0.13  & -0.06  \\
12 & Linear  & 9         &  0.27 &  0.62 &  0.41  & -0.45  & -0.28 & -0.13  & -1.30  & -1.24  \\
13 & RFF     & 3         &  0.39 &  0.54 &  1.32  &  0.83  & -0.00 & -0.05  & -0.04  &  0.02  \\
14 & Linear  & 3         &  0.38 &  0.99 &  0.99  &  0.41  & -0.28 & -0.13  & -1.30  & -1.24  \\
15 & RFF     & Ridgeless &  0.33 &  0.44 &  1.25  &  0.84  &  0.06 &  0.05  &  0.26  &  0.22  \\
16 & RFF     & CV        &  0.29 &  0.37 &  1.10  &  0.73  &  0.08 &  0.16  &  0.18  & -0.03  \\
\midrule
\multicolumn{11}{l}{\textbf{Panel C: 1931--2020}} \\
\midrule
17 & Market  &           &  0.51 &       &        &        &  0.51 &        &        &        \\
18 & History &           &  0.43 &       &  2.70  &        &  0.43 &        &  2.70  &        \\
19 & RFF     & 9         &  0.25 &  0.36 &  0.94  &  0.06  &  0.10 &  0.01  &  0.88  &  0.87  \\
20 & Linear  & 9         &  0.31 &  0.53 &  1.48  & -0.26  & -0.18 & -0.10  & -1.25  & -1.06  \\
21 & RFF     & 3         &  0.46 &  0.50 &  3.02  &  1.81  &  0.12 & -0.00  &  1.15  &  1.14  \\
22 & Linear  & 3         &  0.46 &  0.98 &  2.75  &  0.56  & -0.18 & -0.10  & -1.26  & -1.06  \\
23 & RFF     & Ridgeless &  0.40 &  0.38 &  2.96  &  1.88  &  0.18 &  0.07  &  1.67  &  1.45  \\
24 & RFF     & CV        &  0.35 &  0.31 &  2.09  &  1.36  &  0.13 &  0.09  &  0.99  &  0.71  \\
\bottomrule
\end{tabular}}{\textbf{Note:} Market denotes the equity market portfolio. History denotes the historical average benchmark. RFF denotes the nonlinear model with 12{,}000 random Fourier features. Linear denotes the linear model of the 15 original predictor variables. Non-centered Predictors (KMZ): neither features nor response are demeaned ($\mathrm{dX}=0$, $\mathrm{dY}=0$). Centered Predictors: features are demeaned but response is not ($\mathrm{dX}=1$, $\mathrm{dY}=0$). $k$ denotes the regularization exponent ($z = 10^k$); Ridgeless denotes the ridgeless estimator ($z = 0$); CV denotes cross-validation selected $z$. $t$-statistics use Newey--West standard errors with four lags.}
\end{table}

To explain this puzzling similarity, we must recognize that by forcing the regression through the origin on non-centered variables, KMZ artificially constrain their model. Because the true benchmark under the null of zero slopes requires an intercept, this misspecification mathematically forces the uncentered slope coefficients to absorb the historical average---a problem severely exacerbated when persistent predictors exhibit minimal variation over a short 12-month training window. Table~\ref{tab:12month_center} empirically validates this mechanism: once predictors are appropriately centered, the model can no longer implicitly synthesize an intercept, causing its market-timing performance to evaporate. This evidence substantiates the conceptual critique by \citet{Berk2023} regarding the omitted intercept. We demonstrate that, even though KMZ's main findings survive the inclusion of an intercept, the RFF model itself offers no genuine market-timing ability beyond the historical average benchmark.

A related concern involves their choice of benchmark. KMZ evaluate time-series momentum---defined as the historical average divided by its standard deviation---and conclude that it does not explain the RFF model's performance. We show that this choice is not entirely appropriate under the null of zero slopes and prove that the simple historical average, a standard benchmark widely used in previous studies, provides a theoretically more relevant comparison. Furthermore, \citet{Kelly2025} argues that the volatility-managed momentum proposed by \citet{Nagel2025} is an inadequate benchmark because it is an \textit{ex post} distillation of KMZ's findings rather than a pre-existing strategy. Our results demonstrate that the simple historical average matches the performance of volatility-managed momentum in explaining the RFF model's market-timing ability, successfully overcoming Kelly's critique without relying on any \textit{ex post} construction.

Several other contemporaneous studies also scrutinize KMZ's main empirical findings. \citet{Buncic2025} demonstrate that the virtue of complexity vanishes when (i) an intercept is included in the ridge regressions and (ii) forecasts across random Fourier feature draws are aggregated using an alternative method. \citet{Cartea2025} and \citet{Fallahgoul2025} clarify the virtue of complexity by focusing on noisy predictors and small sample sizes, respectively. While \citet{Kelly2025} provide detailed responses to these critiques, our paper offers new insights into this ongoing debate. By illustrating the limitations of complex machine learning models in equity market return prediction, we clarify a common misinterpretation of benign overfitting in this literature. Although double descent is a genuine statistical phenomenon, it does not imply the virtue of complexity---namely, that big models inherently outperform small models. Under the null hypothesis of zero slope coefficients, the complex model does not outperform the simple intercept-only model. Our evidence shows that this is indeed the case, indicating that the RFF model does not possess genuine out-of-sample forecasting power.

\section{Conclusion}\label{Conclusion}

The recent success of highly parameterized machine learning models in many fields has motivated their application to the central asset-pricing problem of market return predictability. We show that the prevailing intuition of benign overfitting from modern machine learning must be interpreted carefully in this setting characterized by notoriously weak signals. Although the ridgeless random Fourier features estimator exhibits the familiar double-descent phenomenon, this pattern does not necessarily imply a virtue of complexity, where large models inherently outperform small ones. 

Specifically, when the slope coefficients of these features are zero, our theory implies that the optimal ridge penalty satisfies $z^\ast=\infty$. Hence, the asymptotically best large model converges to the historical average---a standard benchmark that has proven difficult to beat in the existing literature. Our empirical findings support this scenario. The RFF model closely tracks the historical average and exhibits negligible incremental market-timing ability. In this sense, our paper reconciles modern and classical statistical learning in asset pricing: overparameterization can generate benign-overfitting and double-descent patterns, but when the predictors carry little true signal, the best regularized large model shrinks back to the simplest forecast---the historical average.

Importantly, our results do not suggest that equity market risk premia are constant, nor do they reject the broader promise of machine learning in asset pricing. The data-generating process of equity market returns is far more complicated than in fields where highly parameterized models typically exhibit superior performance. For instance, combination forecasts have been shown to consistently outperform the historical average benchmark \citep{rapach2010out}, and underlying predictive relationships are inherently time-varying as they critically depend on investors' beliefs. Machine learning architectures explicitly tailored to capture these specific dynamics of equity market returns might still yield significant out-of-sample predictive power.

\begingroup \parindent 0pt \parskip 0.0ex \def\enotesize{\normalsize} \theendnotes \endgroup

%
%
%
%
%

\subsection*{Acknowledgments}
We are grateful to the reviewers of STAI-X conferecne and Yichen Qin for valuable comments and suggestions. We thank Bryan Kelly, Semyon Malamud, Kangying Zhou for sharing code.

\newpage



\bibliographystyle{informs2014} 
\bibliography{sample} 

@article{bartlett2020benign,
  title={Benign overfitting in linear regression},
  author={Bartlett, Peter L and Long, Philip M and Lugosi, Gábor and Tsigler, Alexander},
  journal={Proceedings of the National Academy of Sciences},
  volume={117},
  number={48},
  pages={30063-30070},
  year={2020}
}

@unpublished{Berk2023,
  author  = {Jonathan B. Berk},
  title   = {Comment on “The Virtue of Complexity in Return Prediction”},
  note    = {Working Paper, Stanford University},
  year    = {2023},
}

@unpublished{Cartea2025,
  author  = {Cartea, Alvaro and Jin, Qi and Shi, Yuantao},
  title   = {The Limited Virtue of Complexity in a Noisy World},
  note    = {Working Paper, University of Oxford},
  year    = {2025},
}

@unpublished{Fallahgoul2025,
  author  = {Fallahgoul, Hasan},
  title   = {High-Dimensional Learning in Finance},
  note    = {Working Paper, Monash University},
  year    = {2025},
}

@unpublished{Nagel2025,
  author  = {Nagel, Stefan},
  title   = {Seemingly Virtuous Complexity in Return Prediction},
  note    = {Working Paper, Booth School of Business},
  year    = {2025}
}

@unpublished{Buncic2025,
  author  = {Buncic, Daniel},
  title   = {Simplified: A Closer Look at the Virtue of Complexity in Return Prediction},
  note    = {Working Paper, Stockholm University},
  year    = {2025}
}

@book{Hastie2009,
  title={The Elements of Statistical Learning: Data Mining, Inference, and Prediction},
  author={Hastie, Trevor and Tibshirani, Robert and Friedman, Jerome},
  year={2009},
  publisher={Springer},
  edition={2nd}
}

@article{belkin2019reconciling,
  title={Reconciling modern machine-learning practice and the classical bias--variance trade-off},
  author={Belkin, Mikhail and Hsu, Daniel and Ma, Siyuan and Mandal, Saurabh},
  journal={Proceedings of the National Academy of Sciences},
  volume={116},
  number={32},
  pages={15849--15854},
  year={2019}
}

@article{hastie2022surprises,
  title={Surprises in high-dimensional ridgeless least squares interpolation},
  author={Hastie, Trevor and Montanari, Andrea and Rosset, Saharon and Tibshirani, Ryan J},
  journal={The Annals of Statistics},
  volume={50},
  number={2},
  pages={949--986},
  year={2022}
}

@article{tsigler2023benign,
  title = {Benign Overfitting in Ridge Regression},
  author = {Tsigler, Alexander and Bartlett, Peter L.},
  journal = {Journal of Machine Learning Research},
  volume = {24},
  pages = {123--131},
  year = {2023}
}

@article{mei2022generalization,
  title={The generalization error of random features regression: Precise asymptotics and double descent curve},
  author={Mei, Song and Montanari, Andrea and Nguyen, Phan-Minh},
  journal={Annals of Statistics},
  volume={50},
  number={3},
  pages={1706--1734},
  year={2022}
}

@inproceedings{Rahimi2007,
 author = {Rahimi, Ali and Recht, Benjamin},
 booktitle = {Advances in Neural Information Processing Systems},
 editor = {J. Platt and D. Koller and Y. Singer and S. Roweis},
 pages = {},
 publisher = {Curran Associates, Inc.},
 title = {Random Features for Large-Scale Kernel Machines},
  volume = {20},
 year = {2007}
}

@article{Kelly2024,
author = {Kelly, Bryan and Malamud, Semyon and Zhou, Kangying},
title = {The Virtue of Complexity in Return Prediction},
journal = {The Journal of Finance},
volume = {79},
number = {1},
pages = {459-503},
year = {2024}
}

@techreport{Kelly2025,
author = {Kelly, Bryan and Malamud, Semyon},
title = {Understanding the Virtue of Complexity},
  note    = {Working Paper, Yale University},
year = {2025}
}

@article{campbell2008predicting,
  title={Predicting excess stock returns out of sample: Can anything beat the historical average?},
  author={Campbell, John Y. and Thompson, Samuel B.},
  journal={Review of Financial Studies},
  volume={21},
  number={4},
  pages={1509--1531},
  year={2008},
  publisher={Society for Financial Studies}
}

@article{Goyal2024,
    author = {Goyal, Amit and Welch, Ivo and Zafirov, Athanasse},
    title = {A Comprehensive 2022 Look at the Empirical Performance of Equity Premium Prediction},
    journal = {The Review of Financial Studies},
    volume = {37},
    number = {11},
    pages = {3490-3557},
    year = {2024}
}

@article{rapach2010out,
  title={Out-of-sample equity premium prediction: Combination forecasts and links to the real economy},
  author={Rapach, David E. and Strauss, Jack K. and Zhou, Guofu},
  journal={Review of Financial Studies},
  volume={23},
  number={2},
  pages={821--862},
  year={2010},
  publisher={Society for Financial Studies}
}

@BOOK{HT1990,
  AUTHOR = {Hastie, T. J. and Tibshirani, R. J.},
  TITLE = {Generalized additive models},
  YEAR = {1990},
  PAGES = {335},
  ISBN = {0412343908},
  PUBLISHER = {London: Chapman \& Hall}
}

@article{patil2021uniform,
  title={Uniform consistency of cross-validation estimators for high-dimensional ridge regression},
  author={Patil, Pratik and Wei, Yuting and Rinaldo, Alessandro and Tibshirani, Ryan J.},
  journal={Proceedings of the 24th International Conference on Artificial Intelligence and Statistics},
  volume={130},
  pages={3178--3186},
  year={2021},
  publisher={PMLR}
}

@article{golub1979generalized,
  author    = {Golub, Gene H. and Heath, Michael and Wahba, Grace},
  title     = {Generalized Cross-Validation as a Method for Choosing
               a Good Ridge Parameter},
  journal   = {Technometrics},
  year      = {1979},
  volume    = {21},
  number    = {2},
  pages     = {215--223}
}

@article{Dow1920,
  author  = {Dow, Charles H.},
  title   = {Scientific Stock Speculation},
  journal = {The Magazine of Wall Street},
  year    = {1920}
}

@article{Sharpe1966,
  author  = {Sharpe, William F.},
  title   = {Mutual Fund Performance},
  journal = {The Journal of Business},
  year    = {1966},
  volume  = {39},
  number  = {1},
  pages   = {119--138},
  doi     = {10.1086/294846}
}

@article{Sharpe1994,
  author  = {Sharpe, William F.},
  title   = {The Sharpe Ratio},
  journal = {The Journal of Portfolio Management},
  year    = {1994},
  volume  = {21},
  number  = {1},
  pages   = {49--58}
}

@article{Efron2004,
  author  = {Bradley Efron},
  title   = {The Estimation of Prediction Error: Covariance Penalties and Cross-Validation},
  journal = {Journal of the American Statistical Association},
  year    = {2004},
  volume  = {99},
  number  = {467},
  pages   = {619--632},
  doi     = {10.1198/016214504000000692}
}

@article{newey1994automatic,
  title={Automatic lag selection in covariance matrix estimation},
  author={Newey, Whitney K and West, Kenneth D},
  journal={The review of economic studies},
  volume={61},
  number={4},
  pages={631--653},
  year={1994},
  publisher={Wiley-Blackwell}
}

@article{Welch2008Comprehensive,
  title={A comprehensive look at the empirical performance of equity premium prediction},
  author={Goyal, Amit and Welch, Ivo},
  journal   = {The Review of Financial Studies},
  volume={21},
  number={4},
  pages={1455--1508},
  year={2008},
  publisher={Society for Financial Studies}
}

@article{kunsch1989jackknife,
  title={The jackknife and the bootstrap for general stationary observations},
  author={K{\"u}nsch, Hans R},
  journal={The annals of Statistics},
  pages={1217--1241},
  year={1989},
  publisher={JSTOR}
}

\newpage

\appendix
%
\setcounter{table}{0}

\renewcommand{\thetable}{A\arabic{table}}

\renewcommand{\theHtable}{A\arabic{table}}

\setcounter{figure}{0}
\renewcommand{\thefigure}{A\arabic{figure}}
\renewcommand{\theHfigure}{A\arabic{figure}}

\section{Assumptions and Proofs}

\subsection{Assumption and Proof of Theorem~\ref{theory1}: High-dimensional \texorpdfstring{$P$}{P} Asymptotics}
\label{AppendixA1}

We first provide Assumption and Proof of Theorem~\ref{theory1}: High-dimensional $P$ Asymptotics. In all theories and proofs, we follow the convention working under the population-centered model (\ref{modelpc0}) as in \cite{Kelly2024}, \cite{hastie2022surprises} \cite{patil2021uniform} and omit subscript ``pc'' throughout for notational simplicity.

{\bf[Assumption for Theorem~\ref{theory1} ]}

The predictors vector $\boldsymbol{S}_t \sim P_{\boldsymbol{S}}$ is of the form $\boldsymbol{S}_t = \boldsymbol{\Sigma}^{1/2}\boldsymbol{X}_t$, where defining 
\[
\widehat{H}_T(s) := \frac{1}{P} \sum_{j=1}^P 1_{\{s \ge s_j\}}.
\]
Assume
\begin{itemize}
    \item[(a)] The vector $\boldsymbol{X}_t = (X_{1,t}, \dots, X_{P,t})$ has independent but not necessarily identically distributed elements with $\mathbb{E}[X_{j,t}] = 0$, $\mathbb{E}[X_{j,t}^2] = 1$, and $\mathbb{E}[|X_{j,t}|^k] \le C_k < \infty$ for all $j \le P, k \ge 2$.
    \item[(b)] $s_1 = \|\boldsymbol{\Sigma}\|_{\mathrm{op}} \le M$, $\int s^{-1} d\widehat{H}_T(s) < M$.
    \item[(c)] $|1 - (P/T)| \ge 1/M$, $1/M \le P/T \le M$.
\end{itemize}
Here $M$ and $\{C_k\}$ are large constants. 

\begin{proof} 
We first show the case under the assumption of isotropic features using Corollary 6 of \citet{hastie2022surprises} that the limiting ridge regression risk is minimized at $z^{*}= \gamma \sigma^2/ r^2$, where $r^2 = \| \boldsymbol{\beta} \|^2.$ 

Under the null hypothesis of market return unpredictability, i.e., $\boldsymbol{\beta} = 0$ and $r^2 = 0,$ the ridge regression optimal $$z^{*}=\gamma \sigma^2/ r^2 = \infty.$$ 

Under the optimal $z^{*}=\infty,$ we next show that the model error $ME(z^{*})$, or equivalently the excess prediction error, $PE(z^{*})-\sigma^2,$ we have 
$$
ME(z^{*}) \rightarrow 0 
\quad \text{as} \quad P,T \to \infty \quad \text{with} \quad P/T \to \gamma \in (0,\infty).
$$
Corollary 6 of \citet{hastie2022surprises} gives that $
ME(z^{*}) \rightarrow \gamma \sigma^2 m(-z^{*}),$
where 
$$m(z) = \frac{1 - \gamma - z - \sqrt{(1 - \gamma - z)^2 - 4\gamma z}}{2\gamma z}.$$

We need to  evaluate the limit of $m(-z)$ as $z \to \infty$.
When $z \to \infty,$ $m(-z)$ has the expansion:
$$
m(-z) = \frac{1} {z} - \frac{1} {z^2} + o(\frac{1} {z^2}). 
$$
And $m^{\prime}(-z)$ has the expansion:
$$
m^{\prime}(-z) = \frac{1} {z^2} - \frac{2} {z^3} + o(\frac{1} {z^3}). 
$$
Take the limit as $z \to \infty$, for $\gamma \in (0,\infty) $:
$$\lim_{z \to \infty} m(-z) = \lim_{z \to \infty} \frac{1}{z} = 0, $$
Hence, $$
ME(z^{*}) \rightarrow \gamma \sigma^2 m(-z^{*}) = 0.$$

Note that for a fixed $z>0,$ the variance is not vanishing and $$ME(z) \rightarrow \gamma \sigma^2 (m(-z)-zm^{\prime}(-z)),$$ even though the bias square term vanishes under the null hypothesis of return unpredictability. This is different from the result of classical ridge regression asymptotics that $\|\hat{\boldsymbol{\beta}}_z\|_2 \rightarrow 0$ and $ME(z) \rightarrow 0$ for a fixed $z>0$ under the classical ridge regression with finite $P.$

We then show the case under general covariance matrix $\Sigma$ with spectral distribution $H$ features. 
Denote the Stieltjes transform:
\[
    m(-z) = \int \frac{1}{s + z} \, dH(s), \quad m'(-z) = \int \frac{1}{(s + z)^2} \, dH(s).
\]

Using Theorem 6 of \citet{hastie2022surprises}, the asymptotic bias term $B(z)$ vanishes when $z \to \infty $ as $r^2=0$. The asymptotic variance term is:
\[
    V(z) = \gamma \sigma^2 \int \frac{s^2 (1 - \gamma + \gamma z^2 m'(-z))}{[z + s(1 - \gamma + \gamma z m(-z))]^2} \, dH(s).
\]

We use the expansions for the Stieltjes transform for large $z$, assuming $H$ has finite moments. Substitute into the variance formula. The term inside the parenthesis in the numerator is
$$
1 - \gamma + \gamma z^2 m'(-z)
= 1  -\frac{2\gamma} {z} + o\left(\frac{1}{z}\right) .
$$
The item in the parenthesis in the denominator is
$$
z + s(1 - \gamma + \gamma z m(-z)) = z + s\left(1 - \frac {\gamma} {z} + o\left(\frac{1}{z}\right)\right).
$$
Hence,
$
    \lim_{z \to \infty} V(z) = 0.
$

Therefore, under the null hypothesis of market return unpredictability, at the optimal $z^{*}=\infty,$ the model error tends to zero, that is, 
$ME(z^{*}) \to 0,$ and equivalently the prediction error $PE(z^{*}) \to \sigma^2.$
\end{proof}

\subsection{Assumption, Sketch and Remark of Proof of Theorem~\ref{thm:cv}: CV Uniform Consistency}
\label{AppendixA2}

The following assumptions are used to establish the asymptotic optimal convergence of Theorem~\ref{thm:cv} closely following \cite{patil2021uniform}. That is, GCV-tuned ridge shrinkage parameter yields the same asymptotic prediction error
as the optimally-tuned ridge estimator. 
    \begin{itemize}
		\item[(A1)]
There exists a true coefficient vector $\boldsymbol{\beta}_0 \in \mathbb{R}^P$ such that
$\mathbf{R} = \mathbf{S}\boldsymbol{\beta}_0 + \boldsymbol{\varepsilon}$, where
$\boldsymbol{\varepsilon}$ is independent of $\mathbf{S}$ with i.i.d.\
components satisfying $\mathbb{E}[\varepsilon_t] = 0$,
$\mathbb{E}[\varepsilon_t^2] = \sigma^2 > 0$, and
$\mathbb{E}[|\varepsilon_t|^{4+\delta}] < \infty$ for some $\delta > 0$.

\item[(A2)]
The rows of $\mathbf{S}$ satisfy $S_t = \boldsymbol{\Sigma}^{1/2}\mathbf{X}_t$,
where $\boldsymbol{\Sigma} \in \mathbb{R}^{P\times P}$ is a deterministic
positive definite matrix and $\mathbf{a}_t \in \mathbb{R}^P$ has i.i.d.\
components with mean $0$, variance $1$, and finite $(4+\delta)$th moment.

\item[(A3)]
The eigenvalues of $\boldsymbol{\Sigma}$ satisfy $0 < r_{\min} \leq
\lambda_{\min}(\boldsymbol{\Sigma}) \leq \lambda_{\max}(\boldsymbol{\Sigma})
\leq r_{\max} < \infty$, uniformly in $P$.

\item[(A4)]
$\|\boldsymbol{\beta}_0\|_2^2 \leq C$ for some constant $C < \infty$
independent of $P$.
\end{itemize}

For ridge regression, the Sherman--Morrison--Woodbury identity yields a shortcut formula:
\begin{equation}
\mathrm{CV}(z) \;=\; \frac{1}{T}\sum_{t=0}^{T-1}
\left(\frac{R_{c,t+1} - {\mathbf{S}_{c,t}}\hat{\boldsymbol{\beta}}_z}{1 - \mathbf{H}_{c,ii}(z)}\right)^{\!2},
\label{eq:loocv}
\end{equation}
where $\mathbf{H}_{c,ii}(z)$ is the $i$-th diagonal entry of hat matrix $\mathbf{H_c}(z)$. 

GCV approximates CV by replacing each diagonal element of hat matrix with the average or trace $\mathrm{tr}(\mathbf{H_c}(z))/T$, making it more numerically stable \citep{golub1979generalized}.  
\begin{align}
\mathrm{GCV}(z) \;&=\;
\frac{1}{T} \sum_{t=0}^{T-1}
\left(\frac{R_{c,t+1} - S_{c,t}^\top\prime{\boldsymbol{\beta}}_z}
{1 - \mathrm{tr}(\mathbf{H_c}(z))/T}\right)^{\!2}
\;\notag \\
&=\;
\frac{\dfrac{1}{T}\displaystyle\sum_{t=0}^{T-1}
\!\left(R_{c,t+1} - S_{c,t}^\prime\hat{\boldsymbol{\beta}}_z\right)^{\!2}}
{\left(1 - \mathrm{tr}(\mathbf{H_c}(z))/T\right)^{\!2}},
\label{eq:gcv}
\end{align}
where the numerator is the average squared in-sample residual and
$\mathrm{tr}(\mathbf{H_c}(z))$ is commonly termed as the ``effective degrees of freedom,'' a data-driven measure of model
complexity that increases as $z$ decreases toward zero. The GCV-selected parameter is 
$
\hat{z}_{\mathrm{GCV}} =
\arg\min_{z \in \mathcal{Z}}\,\mathrm{GCV}(z). 
$

Following \citet{patil2021uniform}, we can first establish the GCV-tuned ridge estimator is asymptotically optimal. \\
{\bf{Corollary}}[GCV Uniform Consistency, \citet{patil2021uniform}]
Under Assumptions A1 to A4, let $P, T \to \infty$
with $P/T \to \gamma \in (0,\infty)$. For every compact interval $\mathcal{Z}
\subseteq (z_{\min}, \infty)$,
\begin{equation}
\sup_{z \in \mathcal{Z}}\,\bigl|\mathrm{GCV}(z) - \mathrm{PE}(z)\bigr|
\;\xrightarrow{a.s.}\; 0.
\end{equation}
Consequently, the GCV-tuned estimator achieves the same asymptotic prediction error as the optimally-tuned estimator:
\begin{equation}
\mathrm{PE}(\hat{z}_{\mathrm{GCV}}) - \mathrm{PE}(z^*)
\;\xrightarrow{a.s.}\; 0.
\label{eq:gcv_optimal}
\end{equation}
Again, following \citet{patil2021uniform}, we have
$$CV(z) - GCV(z) \xrightarrow{a.s.} 0$$ 
uniformly for compact intervals $\mathcal{Z}
\subseteq (z_{\min}, \infty)$,
Combining the above asymptotic properties, we obtain the CV uniform consistency of Theorem~\ref{thm:cv}.

\section{Predictor Variables}
\label{AppendixB}

Following KMZ, we employ the fourteen predictor variables from \citet{Welch2008Comprehensive} alongside the lagged excess equity market return. The dividend-price ratio (\texttt{dp}) is the log of the trailing 12-month sum of dividends minus the log of the current index level, while the dividend yield (\texttt{dy}) replaces the current index level with its lagged value. The earnings-price ratio (\texttt{ep}) is the log of the trailing 12-month sum of earnings minus the log of the current index level, and the dividend-payout ratio (\texttt{de}) is the log of trailing dividends minus the log of trailing earnings. These four variables are constructed using data from the S\&P 500 index. The book-to-market ratio (\texttt{b/m}) is the ratio of book value to market value of the equity for the Dow Jones Industrial Average. The Treasury bill rate (\texttt{tbl}) represents the three-month Treasury bill yield from the secondary market, and the long-term yield (\texttt{lty}) is the yield on long-term government bonds, with the term spread (\texttt{tms}) defined as the difference between the two. The default yield spread (\texttt{dfy}) is the difference between BAA- and AAA-rated corporate bond yields, and the default return spread (\texttt{dfr}) is the difference between the returns on long-term corporate and government bonds. We also include the long-term government bond return (\texttt{ltr}). The monthly inflation rate (\texttt{infl}) is computed from the not seasonally adjusted headline Consumer Price Index with a one-month reporting lag. Lastly, stock market variance (\texttt{svar}) is the monthly sum of squared daily returns on the S\&P 500 index, and net equity expansion (\texttt{ntis}) is the ratio of the trailing 12-month sum of net equity issues by NYSE-listed stocks to the total NYSE market capitalization.

\section{Additional Tables and Figures}\label{AppendixC}
This Section provides additional tables and figures for training sample lengths of 60 and 120 months.


\begin{figure}[h]
\centering 
\includegraphics[width=0.8\textwidth]{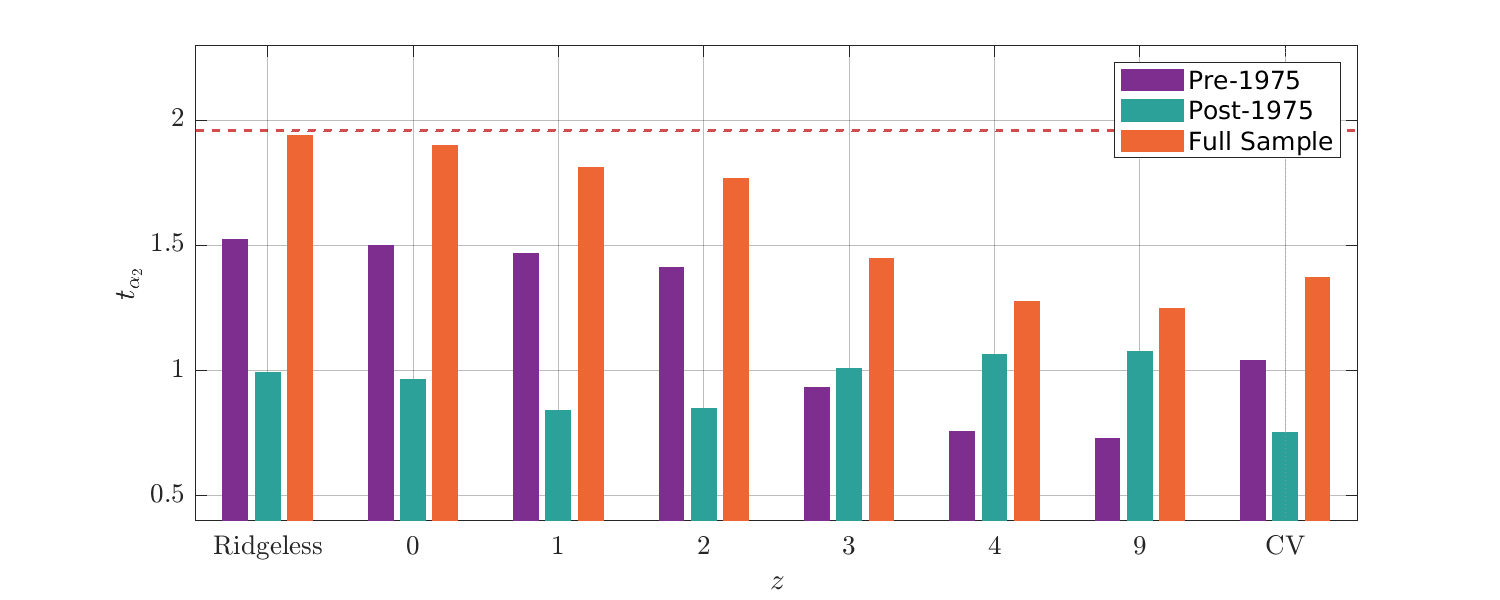}
\caption{Newey--West $t$-Statistic of $\alpha_2$ for the RFF Model across regularization values: 60-Month Training Windows}
\label{fig:alpha2_tstatd_60}
\end{figure}


\begin{figure}[h]
\centering 
\includegraphics[width=0.8\textwidth]{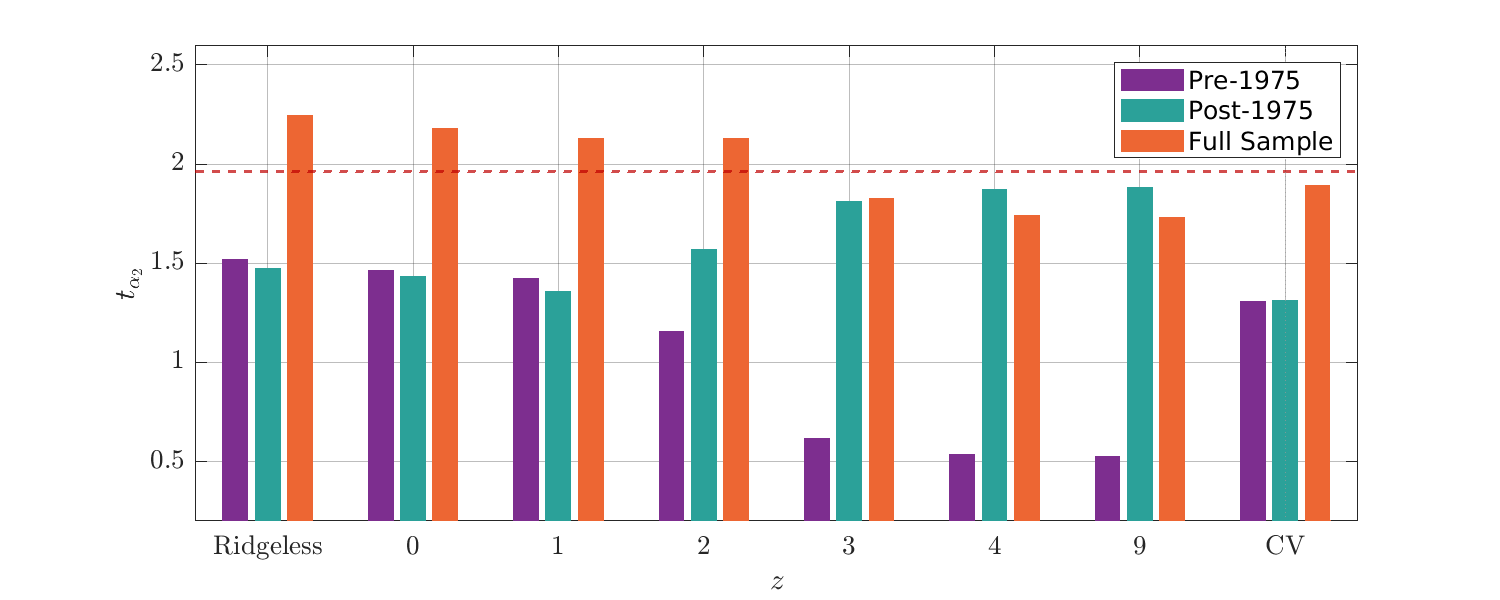}
\caption{Newey--West $t$-Statistic of $\alpha_2$ for the RFF Model across regularization values: 120-Month Training Windows}
\label{fig:alpha2_tstatd_120}
\end{figure}

\begin{table}
\TABLE
{Market-Timing Performance: 60-Month Training Windows\label{tab:60month_s1}}
{%
\begin{tabular}{l L C R R R R R}
\toprule
& Portfolio & $k$ & $\text{OOS } R^2$ & SR & Corr & $t_{\alpha_1}$ & $t_{\alpha_2}$ \\
& & $(z = 10^k)$ & & & & & \\
\midrule
\multicolumn{8}{l}{\textbf{Panel A: 1935--1974}} \\
\midrule
1  & Market  &           &        & 0.52 &      &        &       \\
2  & History &           &        & 0.58 &      &  1.44  &       \\
3  & RFF     & 9         &  0.00  & 0.58 & 1.00 &  1.44  &  0.73 \\
4  & Linear  & 9         &  0.00  & 0.58 & 1.00 &  1.44  &  1.88 \\
5  & RFF     & 3         &  0.00  & 0.60 & 0.98 &  1.63  &  0.93 \\
6  & Linear  & 3         &  0.00  & 0.58 & 1.00 &  1.45  &  1.88 \\
7  & RFF     & Ridgeless & -0.02  & 0.54 & 0.65 &  2.20  &  1.53 \\
8  & RFF     & CV        & -0.00  & 0.56 & 0.76 &  1.85  &  1.04 \\
\midrule
\multicolumn{8}{l}{\textbf{Panel B: 1975--2020}} \\
\midrule
9  & Market  &           &        & 0.56 &      &        &       \\
10 & History &           &        & 0.34 &      & -0.37  &       \\
11 & RFF     & 9         &  0.00  & 0.34 & 1.00 & -0.37  &  1.08 \\
12 & Linear  & 9         &  0.00  & 0.34 & 1.00 & -0.37  &  1.00 \\
13 & RFF     & 3         &  0.00  & 0.36 & 0.99 & -0.19  &  1.01 \\
14 & Linear  & 3         &  0.00  & 0.34 & 1.00 & -0.35  &  1.00 \\
15 & RFF     & Ridgeless & -0.00  & 0.38 & 0.81 &  0.47  &  1.00 \\
16 & RFF     & CV        & -0.00  & 0.37 & 0.93 &  0.08  &  0.75 \\
\midrule
\multicolumn{8}{l}{\textbf{Panel C: 1935--2020}} \\
\midrule
17 & Market  &           &        & 0.54 &      &        &       \\
18 & History &           &        & 0.45 &      &  0.75  &       \\
19 & RFF     & 9         &  0.00  & 0.45 & 1.00 &  0.75  &  1.25 \\
20 & Linear  & 9         &  0.00  & 0.45 & 1.00 &  0.75  &  1.85 \\
21 & RFF     & 3         &  0.00  & 0.47 & 0.99 &  1.05  &  1.45 \\
22 & Linear  & 3         &  0.00  & 0.45 & 1.00 &  0.78  &  1.85 \\
23 & RFF     & Ridgeless & -0.01  & 0.45 & 0.72 &  1.98  &  1.94 \\
24 & RFF     & CV        & -0.00  & 0.46 & 0.84 &  1.51  &  1.37 \\
\bottomrule
\end{tabular}%
}{\textbf{Note:} Market denotes the equity market portfolio. History denotes the historical average benchmark. RFF denotes the nonlinear model with 12{,}000 random Fourier features. Linear denotes the linear model of the 15 original predictor variables. Both features and response are demeaned prior to estimation ($\mathrm{dX}=1$, $\mathrm{dY}=1$). $k$ denotes the regularization exponent ($z = 10^k$); Ridgeless denotes the ridgeless estimator ($z = 0$); CV denotes cross-validation selected $z$. $t$-statistics use Newey--West standard errors with four lags.}
\end{table}

\newpage

\begin{table}
\TABLE
{Market-Timing Performance: 120-Month Training Windows\label{tab:120month_s1}}
{%
\begin{tabular}{l L C R R R R R}
\toprule
& Portfolio & $k$ & $\text{OOS } R^2$ & SR & Corr & $t_{\alpha_1}$ & $t_{\alpha_2}$ \\
& & $(z = 10^k)$ & & & & & \\
\midrule
\multicolumn{8}{l}{\textbf{Panel A: 1940--1974}} \\
\midrule
1  & Market  &           &        & 0.53 &      &       &       \\
2  & History &           &        & 0.70 &      &  2.12 &       \\
3  & RFF     & 9         &  0.00  & 0.70 & 1.00 &  2.12 &  0.53 \\
4  & Linear  & 9         &  0.00  & 0.70 & 1.00 &  2.12 &  1.94 \\
5  & RFF     & 3         &  0.00  & 0.71 & 0.99 &  2.20 &  0.62 \\
6  & Linear  & 3         &  0.00  & 0.70 & 1.00 &  2.14 &  1.95 \\
7  & RFF     & Ridgeless & -0.01  & 0.61 & 0.61 &  2.60 &  1.52 \\
8  & RFF     & CV        &  0.01  & 0.68 & 0.73 &  2.49 &  1.31 \\
\midrule
\multicolumn{8}{l}{\textbf{Panel B: 1975--2020}} \\
\midrule
9  & Market  &           &        & 0.56 &      &       &       \\
10 & History &           &        & 0.41 &      &  0.13 &       \\
11 & RFF     & 9         &  0.00  & 0.41 & 1.00 &  0.13 &  1.88 \\
12 & Linear  & 9         &  0.00  & 0.41 & 1.00 &  0.13 &  0.51 \\
13 & RFF     & 3         &  0.00  & 0.44 & 1.00 &  0.35 &  1.81 \\
14 & Linear  & 3         &  0.00  & 0.41 & 1.00 &  0.13 &  0.52 \\
15 & RFF     & Ridgeless &  0.00  & 0.46 & 0.75 &  1.25 &  1.47 \\
16 & RFF     & CV        &  0.00  & 0.45 & 0.89 &  0.86 &  1.31 \\
\midrule
\multicolumn{8}{l}{\textbf{Panel C: 1940--2020}} \\
\midrule
17 & Market  &           &        & 0.55 &      &       &       \\
18 & History &           &        & 0.54 &      &  1.49 &       \\
19 & RFF     & 9         &  0.00  & 0.54 & 1.00 &  1.49 &  1.73 \\
20 & Linear  & 9         &  0.00  & 0.54 & 1.00 &  1.49 &  1.97 \\
21 & RFF     & 3         &  0.00  & 0.56 & 0.99 &  1.74 &  1.83 \\
22 & Linear  & 3         &  0.00  & 0.55 & 1.00 &  1.50 &  1.98 \\
23 & RFF     & Ridgeless & -0.01  & 0.52 & 0.66 &  2.75 &  2.25 \\
24 & RFF     & CV        &  0.00  & 0.56 & 0.80 &  2.45 &  1.89 \\
\bottomrule
\end{tabular}%
}{\textbf{Note:} Market denotes the equity market portfolio. History denotes the historical average benchmark. RFF denotes the nonlinear model with 12{,}000 random Fourier features. Linear denotes the linear model of the 15 original predictor variables. Both features and response are demeaned prior to estimation ($\mathrm{dX}=1$, $\mathrm{dY}=1$). $k$ denotes the regularization exponent ($z = 10^k$); Ridgeless denotes the ridgeless estimator ($z = 0$); CV denotes cross-validation selected $z$. $t$-statistics use Newey--West standard errors with four lags.}
\end{table}

\newpage

\begin{table}
\TABLE
{Historical Average Benchmark and Momentum Strategies: 60-Month Training Windows\label{tab:mom60month}}
{%
\begin{tabular}{lLRRRR}
\toprule
& Portfolio & SR & Corr & $t_{\alpha_1}$ & $t_{\alpha_2}$ \\
\midrule
\multicolumn{6}{l}{\textbf{Panel A: 1935--1974}} \\
\midrule
1  & History &  0.58 &      &  1.48 &        \\
2  & TsMom   &  0.57 & 1.00 &  1.38 & -1.44  \\
3  & VolMom  &  0.50 & 0.89 &  1.21 & -0.22  \\
4  & DW      &  0.55 & 0.94 &  1.54 &  0.25  \\
5  & PV      &  0.49 & 0.77 &  0.37 & -0.25  \\
6  & EWPV    &  0.56 & 0.93 &  1.49 &  0.22  \\
\midrule
\multicolumn{6}{l}{\textbf{Panel B: 1975--2020}} \\
\midrule
7  & History &  0.33 &      & -0.46 &        \\
8  & TsMom   &  0.34 & 1.00 & -0.38 &  0.58  \\
9  & VolMom  &  0.34 & 0.92 & -0.03 &  0.66  \\
10 & DW      &  0.36 & 0.94 &  0.03 &  0.84  \\
11 & PV      &  0.44 & 0.84 & -0.88 & -0.75  \\
12 & EWPV    &  0.31 & 0.95 & -0.43 & -0.03  \\
\midrule
\multicolumn{6}{l}{\textbf{Panel C: 1935--2020}} \\
\midrule
13 & History &  0.45 &      &  0.79 &        \\
14 & TsMom   &  0.45 & 1.00 &  0.78 & -0.04  \\
15 & VolMom  &  0.42 & 0.91 &  0.87 &  0.33  \\
16 & DW      &  0.45 & 0.94 &  1.11 &  0.89  \\
17 & PV      &  0.47 & 0.79 & -0.13 & -0.58  \\
18 & EWPV    &  0.43 & 0.94 &  0.83 &  0.20  \\
\bottomrule
\end{tabular}%
}{\textbf{Note:} History denotes the historical average benchmark. TsMom denotes time-series momentum. VolMom denotes volatility-scaled momentum. DW denotes the Dynamically Weighted strategy. PV and EWPV denote the Parametric and Equal-Weighted Parametric Volatility strategies. Corr denotes the correlation of each portfolio return with the historical average. $t_{\alpha_1}$ and $t_{\alpha_2}$ denote $t$-statistics for the intercept from regressions on the market and on the market plus historical average, respectively, using Newey--West standard errors with four lags.}
\end{table}

\newpage

\begin{table}
\TABLE
{Historical Average Benchmark and Momentum Strategies: 120-Month Training Windows\label{tab:mom120month}}
{%
\begin{tabular}{lLRRRR}
\toprule
& Portfolio & SR & Corr & $t_{\alpha_1}$ & $t_{\alpha_2}$ \\
\midrule
\multicolumn{6}{l}{\textbf{Panel A: 1940--1974}} \\
\midrule
1  & History &  0.70 &      &  2.15 &        \\
2  & TsMom   &  0.69 & 1.00 &  2.08 & -1.25  \\
3  & VolMom  &  0.60 & 0.95 &  1.37 & -1.28  \\
4  & DW      &  0.66 & 0.97 &  1.83 & -0.50  \\
5  & PV      &  0.45 & 0.83 & -0.60 & -1.17  \\
6  & EWPV    &  0.62 & 0.97 &  1.50 & -1.10  \\
\midrule
\multicolumn{6}{l}{\textbf{Panel B: 1975--2020}} \\
\midrule
7  & History &  0.41 &      &  0.04 &        \\
8  & TsMom   &  0.40 & 1.00 & -0.05 & -1.53  \\
9  & VolMom  &  0.29 & 0.92 & -0.56 & -1.28  \\
10 & DW      &  0.38 & 0.96 & -0.15 & -0.45  \\
11 & PV      &  0.46 & 0.80 & -1.15 & -1.40  \\
12 & EWPV    &  0.30 & 0.94 & -0.52 & -1.29  \\
\midrule
\multicolumn{6}{l}{\textbf{Panel C: 1940--2020}} \\
\midrule
13 & History &  0.54 &      &  1.50 &        \\
14 & TsMom   &  0.54 & 1.00 &  1.43 & -0.80  \\
15 & VolMom  &  0.44 & 0.94 &  0.64 & -1.48  \\
16 & DW      &  0.51 & 0.96 &  1.22 & -0.40  \\
17 & PV      &  0.45 & 0.82 & -1.06 & -1.83  \\
18 & EWPV    &  0.46 & 0.95 &  0.69 & -1.51  \\
\bottomrule
\end{tabular}%
}{\textbf{Note:} History denotes the historical average benchmark. TsMom denotes time-series momentum. VolMom denotes volatility-scaled momentum. DW denotes the Dynamically Weighted strategy. PV and EWPV denote the Parametric and Equal-Weighted Parametric Volatility strategies. Corr denotes the correlation of each portfolio return with the historical average. $t_{\alpha_1}$ and $t_{\alpha_2}$ denote $t$-statistics for the intercept from regressions on the market and on the market plus historical average, respectively, using Newey--West standard errors with four lags.}
\end{table}

\newpage
\clearpage

\begin{table}[p]
\TABLE
{Bootstrap Null Distribution of Out-of-Sample Performance: 60-Month Window\label{tab:bootstrap60}}
{%
\begin{tabular}{l R R R R}
\toprule
& SR & OOS $R^2$ & $t_{\alpha_1}$ & $t_{\alpha_2}$ \\
\midrule
Original Data     & 0.46 & $-$0.00 & 1.54 & 1.41 \\
Percentile        & 19.00 &  16.00 & 87.00 & 0.00 \\
Mean              & 0.47 &    0.00 & 1.44 & 2.13 \\
Median            & 0.47 & $-$0.00 & 1.44 & 2.15 \\
Std               & 0.01 &    0.00 & 0.09 & 0.26 \\
5th Pct           & 0.45 & $-$0.00 & 1.29 & 1.70 \\
95th Pct          & 0.48 &    0.00 & 1.61 & 2.60 \\
\midrule
Pooled Percentile & 39.7 &  37.0 & 29.00 & 23.67 \\
\bottomrule
\end{tabular}%
}{\textbf{Note:} Original Data denotes the out-of-sample performance of the model estimated on the original data. Percentile denotes the percentile of the original data performance within the window-specific bootstrap null distribution. Pooled Percentile reports, for $t_{\alpha_1}$ and $t_{\alpha_2}$, the percentile of the original data statistic within the bootstrap null distribution pooled across all training windows. Mean, Median, Std, 5th Pct, and 95th Pct summarise the distribution of out-of-sample performance across $B=100$ bootstrap replications, each using $100$ independently drawn sets of random Fourier features. SR denotes the annualized Sharpe ratio of the market timing strategy. OOS $R^2$ denotes the out-of-sample $R^2$ of the return forecast. $t_{\alpha_1}$ and $t_{\alpha_2}$ denote the Newey-West $t$-statistics from regressing the strategy return on the market ($\alpha_1$) and on the market plus the historical average ($\alpha_2$).}
\end{table}

\newpage
\clearpage
\begin{table}[p]
\TABLE
{Bootstrap Null Distribution of Out-of-Sample Performance: 120-Month Window\label{tab:bootstrap120}}
{%
\begin{tabular}{l R R R R}
\toprule
& SR & OOS $R^2$ & $t_{\alpha_1}$ & $t_{\alpha_2}$ \\
\midrule
Original Data     & 0.56 &    0.00 & 2.46 & 1.90 \\
Percentile        & 40.00 & 100.00 & 100.00 & 74.00 \\
Mean              & 0.56 & $-$0.00 & 2.04 & 1.72 \\
Median            & 0.56 & $-$0.00 & 2.03 & 1.67 \\
Std               & 0.01 &    0.00 & 0.11 & 0.27 \\
5th Pct           & 0.55 & $-$0.00 & 1.89 & 1.35 \\
95th Pct          & 0.58 &    0.00 & 2.21 & 2.19 \\
\midrule
Pooled Percentile & 80.00 &  100.00 & 66.67 & 62.33 \\
\bottomrule
\end{tabular}%
}{\textbf{Note:} Original Data denotes the out-of-sample performance of the model estimated on the original data. Percentile denotes the percentile of the original data performance within the window-specific bootstrap null distribution. Pooled Percentile reports, for $t_{\alpha_1}$ and $t_{\alpha_2}$, the percentile of the original data statistic within the bootstrap null distribution pooled across all training windows. Mean, Median, Std, 5th Pct, and 95th Pct summarise the distribution of out-of-sample performance across $B=100$ bootstrap replications, each using $100$ independently drawn sets of random Fourier features. SR denotes the annualized Sharpe ratio of the market timing strategy. OOS $R^2$ denotes the out-of-sample $R^2$ of the return forecast. $t_{\alpha_1}$ and $t_{\alpha_2}$ denote the Newey-West $t$-statistics from regressing the strategy return on the market ($\alpha_1$) and on the market plus the historical average ($\alpha_2$).}
\end{table}

\newpage

\begin{table}
\TABLE
{Market-Timing Performance: 60-Month Training Windows with KMZ No Intercept with Non-centered Market Return\label{tab:60month_center}}
{\begin{tabular}{l L C R R R R R R R R}
\toprule
\multicolumn{2}{c}{} & k & \multicolumn{4}{c}{\textbf{Non-centered Predictors (KMZ)}} & \multicolumn{4}{c}{\textbf{Centered Predictors}} \\
\cmidrule(lr){4-7} \cmidrule(lr){8-11}
& Portfolio & $(z = 10^k)$ & SR & Corr & $t_{\alpha_1}$ & $t_{\alpha_2}$
                  & SR & Corr & $t_{\alpha_1}$ & $t_{\alpha_2}$ \\
\midrule
\multicolumn{11}{l}{\textbf{Panel A: 1935--1974}} \\
\midrule
1  & Market  &           &  0.52 &       &        &        &  0.52 &        &        &       \\
2  & History &           &  0.58 &       &  1.44  &        &  0.58 &        &  1.44  &       \\
3  & RFF     & 9         &  0.36 &  0.40 &  1.04  &  0.70  &  0.09 & -0.04  &  0.83  &  0.71 \\
4  & Linear  & 9         &  0.39 &  0.82 &  0.30  & -0.92  &  0.25 & -0.03  &  2.02  &  1.86 \\
5  & RFF     & 3         &  0.43 &  0.41 &  1.59  &  1.21  &  0.13 & -0.01  &  1.05  &  0.93 \\
6  & Linear  & 3         &  0.58 &  0.97 &  1.44  &  0.27  &  0.25 & -0.03  &  2.02  &  1.86 \\
7  & RFF     & Ridgeless &  0.41 &  0.27 &  2.32  &  1.97  &  0.24 &  0.05  &  1.70  &  1.52 \\
8  & RFF     & CV        &  0.34 &  0.15 &  1.56  &  1.39  &  0.19 &  0.00  &  1.23  &  1.11 \\
\midrule
\multicolumn{11}{l}{\textbf{Panel B: 1975--2020}} \\
\midrule
9  & Market  &           &  0.56 &       &        &        &  0.56 &        &        &       \\
10 & History &           &  0.34 &       & -0.37  &        &  0.34 &        & -0.37  &       \\
11 & RFF     & 9         &  0.46 &  0.47 &  1.98  &  2.15  &  0.14 & -0.13  &  1.08  &  1.04 \\
12 & Linear  & 9         &  0.38 &  0.82 &  0.60  &  1.16  &  0.21 & -0.06  &  1.01  &  0.99 \\
13 & RFF     & 3         &  0.45 &  0.44 &  1.90  &  2.04  &  0.14 & -0.12  &  1.03  &  0.99 \\
14 & Linear  & 3         &  0.37 &  0.98 & -0.02  &  1.16  &  0.21 & -0.06  &  1.01  &  0.99 \\
15 & RFF     & Ridgeless &  0.36 &  0.34 &  1.56  &  1.63  &  0.17 & -0.00  &  1.01  &  0.99 \\
16 & RFF     & CV        &  0.28 &  0.29 &  1.24  &  1.28  &  0.14 &  0.05  &  0.94  &  0.93 \\
\midrule
\multicolumn{11}{l}{\textbf{Panel C: 1935--2020}} \\
\midrule
17 & Market  &           &  0.54 &       &        &        &  0.54 &        &        &       \\
18 & History &           &  0.45 &       &  0.75  &        &  0.45 &        &  0.75  &       \\
19 & RFF     & 9         &  0.37 &  0.39 &  1.70  &  1.54  &  0.10 & -0.07  &  1.21  &  1.22 \\
20 & Linear  & 9         &  0.33 &  0.72 &  0.61  &  0.17  &  0.23 & -0.05  &  1.81  &  1.84 \\
21 & RFF     & 3         &  0.42 &  0.40 &  2.30  &  2.16  &  0.13 & -0.05  &  1.42  &  1.43 \\
22 & Linear  & 3         &  0.44 &  0.97 &  0.85  &  0.32  &  0.23 & -0.05  &  1.81  &  1.84 \\
23 & RFF     & Ridgeless &  0.38 &  0.29 &  2.76  &  2.64  &  0.20 &  0.03  &  1.96  &  1.94 \\
24 & RFF     & CV        &  0.31 &  0.21 &  1.96  &  1.86  &  0.16 &  0.02  &  1.52  &  1.48 \\
\bottomrule
\end{tabular}}{\textbf{Note:} Market denotes the equity market portfolio. History denotes the historical average benchmark. RFF denotes the nonlinear model with 12{,}000 random Fourier features. Linear denotes the linear model of the 15 original predictor variables. Non-centered Predictors (KMZ): neither features nor response are demeaned ($\mathrm{dX}=0$, $\mathrm{dY}=0$). Centered Predictors: features are demeaned but response is not ($\mathrm{dX}=1$, $\mathrm{dY}=0$). $k$ denotes the regularization exponent ($z = 10^k$); Ridgeless denotes the ridgeless estimator ($z = 0$); CV denotes cross-validation selected $z$. $t$-statistics use Newey--West standard errors with four lags.}
\end{table}
 
\newpage
 
\begin{table}
\TABLE
{Market-Timing Performance: 120-Month Training Windows with KMZ No Intercept with Non-centered Market Return\label{tab:120month_center}}
{\begin{tabular}{l L C R R R R R R R R}
\toprule
\multicolumn{2}{c}{} & k & \multicolumn{4}{c}{\textbf{Non-centered Predictors (KMZ)}} & \multicolumn{4}{c}{\textbf{Centered Predictors}} \\
\cmidrule(lr){4-7} \cmidrule(lr){8-11}
& Portfolio & $(z = 10^k)$ & SR & Corr & $t_{\alpha_1}$ & $t_{\alpha_2}$
                  & SR & Corr & $t_{\alpha_1}$ & $t_{\alpha_2}$ \\
\midrule
\multicolumn{11}{l}{\textbf{Panel A: 1940--1974}} \\
\midrule
1  & Market  &           &  0.53 &       &        &        &  0.53 &        &        &       \\
2  & History &           &  0.70 &       &  2.12  &        &  0.70 &        &  2.12  &       \\
3  & RFF     & 9         &  0.38 &  0.40 &  1.06  &  0.63  &  0.08 & -0.05  &  1.09  &  0.51 \\
4  & Linear  & 9         &  0.53 &  0.86 &  0.74  & -0.56  &  0.21 & -0.19  &  2.56  &  1.93 \\
5  & RFF     & 3         &  0.42 &  0.41 &  1.37  &  0.87  &  0.10 & -0.03  &  1.14  &  0.61 \\
6  & Linear  & 3         &  0.66 &  0.97 &  1.78  & -0.27  &  0.21 & -0.19  &  2.56  &  1.93 \\
7  & RFF     & Ridgeless &  0.43 &  0.25 &  2.46  &  1.93  &  0.25 &  0.03  &  1.84  &  1.52 \\
8  & RFF     & CV        &  0.36 &  0.15 &  1.74  &  1.42  &  0.23 & -0.01  &  1.66  &  1.37 \\
\midrule
\multicolumn{11}{l}{\textbf{Panel B: 1975--2020}} \\
\midrule
9  & Market  &           &  0.56 &       &        &        &  0.56 &        &        &       \\
10 & History &           &  0.41 &       &  0.13  &        &  0.41 &        &  0.13  &       \\
11 & RFF     & 9         &  0.46 &  0.37 &  2.01  &  2.02  &  0.28 & -0.03  &  1.84  &  1.85 \\
12 & Linear  & 9         &  0.25 &  0.74 & -0.04  & -0.16  &  0.13 & -0.04  &  0.49  &  0.50 \\
13 & RFF     & 3         &  0.45 &  0.36 &  1.96  &  1.97  &  0.28 & -0.02  &  1.79  &  1.80 \\
14 & Linear  & 3         &  0.38 &  0.96 &  0.18  &  0.20  &  0.13 & -0.04  &  0.50  &  0.51 \\
15 & RFF     & Ridgeless &  0.36 &  0.28 &  1.58  &  1.57  &  0.25 &  0.04  &  1.47  &  1.47 \\
16 & RFF     & CV        &  0.27 &  0.23 &  1.23  &  1.22  &  0.18 &  0.07  &  1.18  &  1.18 \\
\midrule
\multicolumn{11}{l}{\textbf{Panel C: 1940--2020}} \\
\midrule
17 & Market  &           &  0.55 &       &        &        &  0.55 &        &        &       \\
18 & History &           &  0.54 &       &  1.49  &        &  0.54 &        &  1.49  &       \\
19 & RFF     & 9         &  0.37 &  0.37 &  1.75  &  1.43  &  0.15 & -0.04  &  1.80  &  1.70 \\
20 & Linear  & 9         &  0.29 &  0.67 &  0.10  & -0.68  &  0.16 & -0.12  &  1.84  &  1.95 \\
21 & RFF     & 3         &  0.40 &  0.38 &  2.12  &  1.80  &  0.17 & -0.03  &  1.89  &  1.82 \\
22 & Linear  & 3         &  0.49 &  0.95 &  1.03  & -0.64  &  0.16 & -0.12  &  1.85  &  1.96 \\
23 & RFF     & Ridgeless &  0.39 &  0.26 &  2.89  &  2.63  &  0.24 &  0.04  &  2.33  &  2.24 \\
24 & RFF     & CV        &  0.31 &  0.18 &  2.11  &  1.94  &  0.20 &  0.02  &  2.01  &  1.93 \\
\bottomrule
\end{tabular}}{\textbf{Note:} Market denotes the equity market portfolio. History denotes the historical average benchmark. RFF denotes the nonlinear model with 12{,}000 random Fourier features. Linear denotes the linear model of the 15 original predictor variables. Non-centered Predictors (KMZ): neither features nor response are demeaned ($\mathrm{dX}=0$, $\mathrm{dY}=0$). Centered Predictors: features are demeaned but response is not ($\mathrm{dX}=1$, $\mathrm{dY}=0$). $k$ denotes the regularization exponent ($z = 10^k$); Ridgeless denotes the ridgeless estimator ($z = 0$); CV denotes cross-validation selected $z$. $t$-statistics use Newey--West standard errors with four lags.}
\end{table}

\end{document}